\PassOptionsToPackage{table}{xcolor}
\documentclass{article} 
\usepackage{arxiv}

\usepackage{amsmath,amsfonts,bm}

\def\eqref#1{equation~\ref{#1}}

\def\1{\bm{1}}

\DeclareMathAlphabet{\mathsfit}{\encodingdefault}{\sfdefault}{m}{sl}
\SetMathAlphabet{\mathsfit}{bold}{\encodingdefault}{\sfdefault}{bx}{n}

\DeclareMathOperator*{\argmin}{arg\,min}

\usepackage[round]{natbib}
\usepackage{hyperref}
\usepackage{url}
\usepackage{xcolor}
\usepackage{booktabs}
\usepackage{longtable}
\usepackage{microtype}
\usepackage{graphicx}
\usepackage{algorithm}
\usepackage{algpseudocode}
\usepackage{amsthm}
\usepackage{capt-of}
\usepackage{float}
\usepackage{placeins}
\usepackage{wrapfig}

\hypersetup{
  hidelinks,
  breaklinks=true,
  hypertexnames=false,
  pdftitle={ASAP: Amortized Doubly-Stochastic Attention via Sliced Dual Projection},
  pdfauthor={Huy Tran, Max Milkert, David Hyde}
}
\makeatletter
\g@addto@macro\UrlBreaks{\do\/\do-\do\_\do\.\do\?\do\&\do\=\do\#}
\makeatother

\newsavebox{\cifarTablePanel}

\newtheorem{proposition}{Proposition}[section]

\title{ASAP: Amortized Doubly-Stochastic Attention via Sliced Dual Projection}

\author{
Huy Tran\thanks{Equal contribution.}\\
Department of Computer Science\\
Vanderbilt University
\And
Max Milkert\footnotemark[1]\\
Department of Computer Science\\
Vanderbilt University
\And
David Hyde\\
Department of Computer Science\\
Vanderbilt University
}

\renewcommand{\headeright}{Preprint}
\renewcommand{\undertitle}{Preprint}

\begin{document}

\maketitle

\begin{abstract}
Doubly-stochastic attention has emerged as a transport-based alternative to row-softmax attention, with recent Transformer variants using it to reduce attention sinks and rank collapse while improving performance. In this family, the standard approach is Sinkhorn scaling, which trains more efficiently but still repeats matrix scaling in every inference forward pass. Sliced-transport attention removes the online iteration, but its soft sorting approximation materializes dense tensors for each slice, requiring substantially more training resources than Sinkhorn attention. We introduce \emph{ASAP} (\emph{\textbf{A}mortized Doubly-\textbf{S}tochastic \textbf{A}ttention via Sliced Dual \textbf{P}rojection}), a train-then-compile method that trains the doubly-stochastic layer with Sinkhorn, then replaces the iterative scaling loop at inference with a fixed sliced-dual operator. It learns a lightweight parametric map from exact one-dimensional Kantorovich potentials to the Sinkhorn query-side dual, then reconstructs the attention plan with a two-sided entropic $c$-transform. Across language and vision benchmarks, ASAP keeps the cheaper training setup and remains highly competitive with recent baselines. In the main frozen-layer benchmark, ASAP is $5.3\times$ faster than the trained Sinkhorn teacher while matching its accuracy; in downstream replacements, ASAP recovers most of the teacher performance without any retraining.
\end{abstract}

\section{Introduction}

\begin{wrapfigure}[13]{r}{0.36\textwidth}
\vspace{-16pt}
\centering
\includegraphics[width=0.35\textwidth]{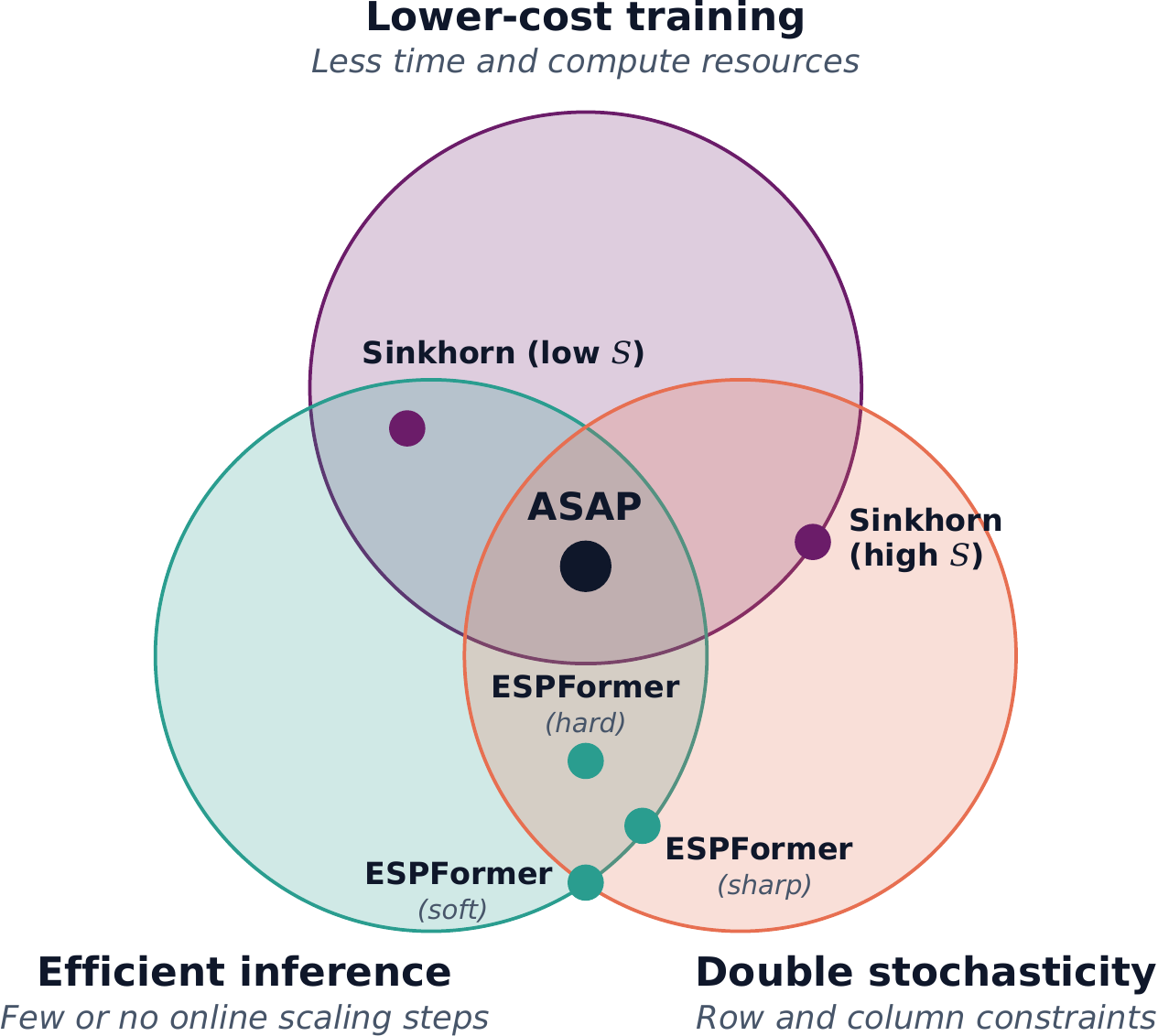}
\vspace{2pt}
\end{wrapfigure}

The Transformer is a standard architecture in modern language and vision systems \citep{vaswani2017attention,dosovitskiy2021image}, and its success has made attention normalization a central design choice. Standard row-softmax attention normalizes each query independently, a one-sided constraint that recent work has associated with attention sinks in language models and rank collapse in deep attention stacks \citep{xiao2024streamingllm,dong2021attention,lapenna2026rankdecay}. Doubly-stochastic attention adds the matching column constraint, providing a transport structure that has been shown to reduce these pathologies and improve performance across recent Transformer variants \citep{sander2022sinkformers,born2025quantum}. Trained row-softmax attention can also move toward this structure during optimization, suggesting that column balance is a useful inductive bias \citep{sander2022sinkformers}.

The seminal Sinkformer work turns this bias into a trainable layer by replacing row-softmax normalization with a finite Sinkhorn scaling loop to obtain an entropic transport plan between query and key tokens \citep{sander2022sinkformers,cuturi2013sinkhorn}. 
The Sinkhorn layer keeps the surrounding Transformer block unchanged, supports end-to-end training through Sinkhorn normalization, and has been shown to improve downstream performance across several benchmarks. However, a nonparametric Sinkhorn normalizer still repeats a finite scaling loop at inference, requiring $O(SN^2)$ scaling work for $N$ tokens and $S$ normalization steps on every call. Reducing $S$ after training can lower latency, but it stops farther from the doubly-stochastic limit and changes the finite normalizer for the fixed projections. The practical question is how to avoid online scaling in a trained Sinkhorn layer while preserving its attention operator.

Other recent work has pursued non-iterative doubly-stochastic attention by changing the transport construction itself. ESPFormer introduces an elegant approach based on Expected Sliced Transport Plans, using sorted one-dimensional projections of queries and keys to form slice-wise matchings that are lifted and averaged across slices \citep{liu2025expected,shahbazi2025espformer}. Its reported results show that sliced-transport attention can remain competitive across tasks, sometimes with improved performance, and its hard-sort inference operator is efficient, non-iterative, and exactly doubly stochastic. That benefit, however, comes after training with a soft sorting approximation, while inference switches to the hard doubly-stochastic matrix. Under the SoftSort operator used in ESPFormer, the training relaxation is row-stochastic at finite temperature and approaches the hard doubly-stochastic matrix only as the temperature is annealed. The soft-sort training operator materializes dense relaxation tensors for each slice, and in our matched IMDb resource profile against the $S=10$ Sinkhorn teacher budget it uses $5.7\times$ more peak training memory and runs $14.8\times$ slower per training example than Sinkhorn attention. The larger training footprint significantly raises the cost of model selection, ablations, and repeated benchmark runs. This interesting line of work therefore leaves open a complementary question of whether non-iterative inference can be reached while keeping the cheaper training setup like Sinkhorn attention.

The two open questions above point to a doubly-stochastic attention trilemma. Current designs can have reasonable training cost, reduce online scaling, or stay close to a high-budget doubly-stochastic operator, but not all three at once. We introduce ASAP as a post-training answer to this trilemma. We instantiate this train-then-compile view with Sinkhorn teachers because they are the less expensive training setup in the current literature, and make the finite dual variables available for compilation. Given a trained, high-performing Sinkhorn layer, we freeze it and run a small unlabeled calibration set through the layer. From each forward pass, we extract the query-side normalization vector that the Sinkhorn loop would otherwise recompute; in optimal-transport terms, this is the source dual. On the same query and key activations, we compute exact one-dimensional transport potentials from sorted projections and use them as sliced geometric features. We then learn a lightweight parametric map from these features to the extracted Sinkhorn vector. At inference, the compiled layer predicts this one vector and applies a two-sided entropic $c$-transform under the teacher's finite Sinkhorn convention. This keeps the cheaper Sinkhorn training setup, replaces most of the online scaling loop, and preserves the matched marginal by construction.

In summary, our specific contributions are as follows.
\vspace{-3pt}
\begin{itemize}
\item We introduce ASAP, a train-then-compile method for doubly-stochastic attention that keeps the Sinkhorn training setup and replaces the iterative online scaling loop after training while staying close to the trained Sinkhorn operator.
\item We instantiate sliced-dual projection for frozen Sinkhorn attention, regressing exact one-dimensional Kantorovich potentials onto the teacher's \textit{finite-budget} source dual and recovering the key dual analytically through the entropic $c$-transform.
\item One frozen teacher supports multiple ASAP variants without retraining. Least-squares or convex-KL can be used to calibrate the sliced-dual map, and a one-sided or two-sided c-transform can be selected at inference.
\item We demonstrate ASAP across multiple language and vision benchmarks, showing significant layer speedups over the trained Sinkhorn teacher with matched-resource comparisons against iterative Sinkhorn and sliced-transport baselines.
\end{itemize}

\section{Related Work}

\paragraph{Transport-based attention.}
Transport-based attention replaces row-softmax normalization with operators that impose additional structure on the attention plan. Sparse Sinkhorn Attention uses Sinkhorn balancing to learn block-wise relaxed permutations for sparse attention \citep{tay2020sparse}, while Sinkformers apply Sinkhorn scaling directly to the query--key score matrix. Other variants explore non-iterative, low-rank, and alternative parameterizations of doubly-stochastic attention, including LOTFormer, which uses a low-rank pivot coupling for linear-time inference \citep{shahbazi2025lotformer}. ASAP stays closest to the Sinkhorn line, keeping the dense trained Sinkhorn operator fixed and replacing only repeated normalization after training. Recent theory studies rank decay in Sinkhorn-style doubly-stochastic attention, while FlashSinkhorn accelerates entropic OT solvers at the kernel level without changing the transport operator \citep{ye2026flashsinkhorn}.

\paragraph{Sliced optimal transport.}
Sliced OT uses the closed-form structure of one-dimensional OT as a computational primitive for high-dimensional measures \citep{nguyen2025intro}. Recent work improves the choice of slices and extends the construction beyond Euclidean measures, including convolutional slices for images, adaptive slicing distributions, quasi-Monte Carlo slice sets, spherical and stereographic spherical sliced distances, Cartan-Hadamard geometry, and unbalanced transport \citep{nguyen2022revisiting,nguyen2023energy,nguyen2024qmc,bonet2023spherical,tran2024stereographic,bonet2025cartan,sejourne2023unbalanced}. Another line lifts one-dimensional slice solutions back into transport objects, including min-SWGG, Sliced OT Plans, generalized sliced Wasserstein plans, and min-STP \citep{mahey2023minswgg,tanguy2025sotp,chapel2025dgswp,liu2025minstp}. Sliced-regularized OT instead uses a smoothed sliced OT plan as the reference coupling in a KL-regularized OT problem, with a Sinkhorn-style solver around an informative sliced prior \citep{nguyen2026srot}. ASAP uses one-dimensional sliced potentials as token-level coordinates for calibrating a frozen Sinkhorn attention layer.

\paragraph{Amortized transport.}
Amortized OT methods learn reusable structure across transport problems. Meta OT predicts OT potentials or maps for repeated problems \citep{amos2023metaot}, and Sinkhorn-initialization methods use data-dependent dual initializers, including closed-form one-dimensional OT duals, to reduce solver time \citep{thornton2023rethinking}. Wasserstein Wormhole learns embeddings whose Euclidean distances approximate Wasserstein distances \citep{haviv2024wassersteinwormhole}, and \citet{nguyen2026fastestimation} regress Wasserstein distances on sliced-Wasserstein features. \citet{truong2026aot} are closest in feature idea, using sliced Kantorovich potentials to amortize repeated OT solves. ASAP follows this sliced-potential construction inside a trained Transformer layer, where the object being compiled is the finite Sinkhorn normalizer with its learned score scale, mask, entropy, and update convention.

\paragraph{Other attention alternatives.}
A parallel literature improves attention without transport constraints. DiffTransformer subtracts two softmax operators to suppress attention noise \citep{ye2025differential}. Sparse, low-rank, kernel, and Nystr\"om approximations address the quadratic cost of vanilla attention \citep{kitaev2020reformer,wang2020linformer,choromanski2021performer,xiong2021nystromformer}, while FlashAttention provides exact IO-aware softmax kernels \citep{dao2022flashattention}.

\section{Background}

\subsection{Doubly-Stochastic Attention via Sinkhorn Scaling}
\label{sec:background_transport_attention}

\begin{wrapfigure}[27]{r}{0.26\textwidth}
\vspace{-10pt}
\centering
\makebox[\linewidth][r]{\includegraphics[width=0.275\textwidth]{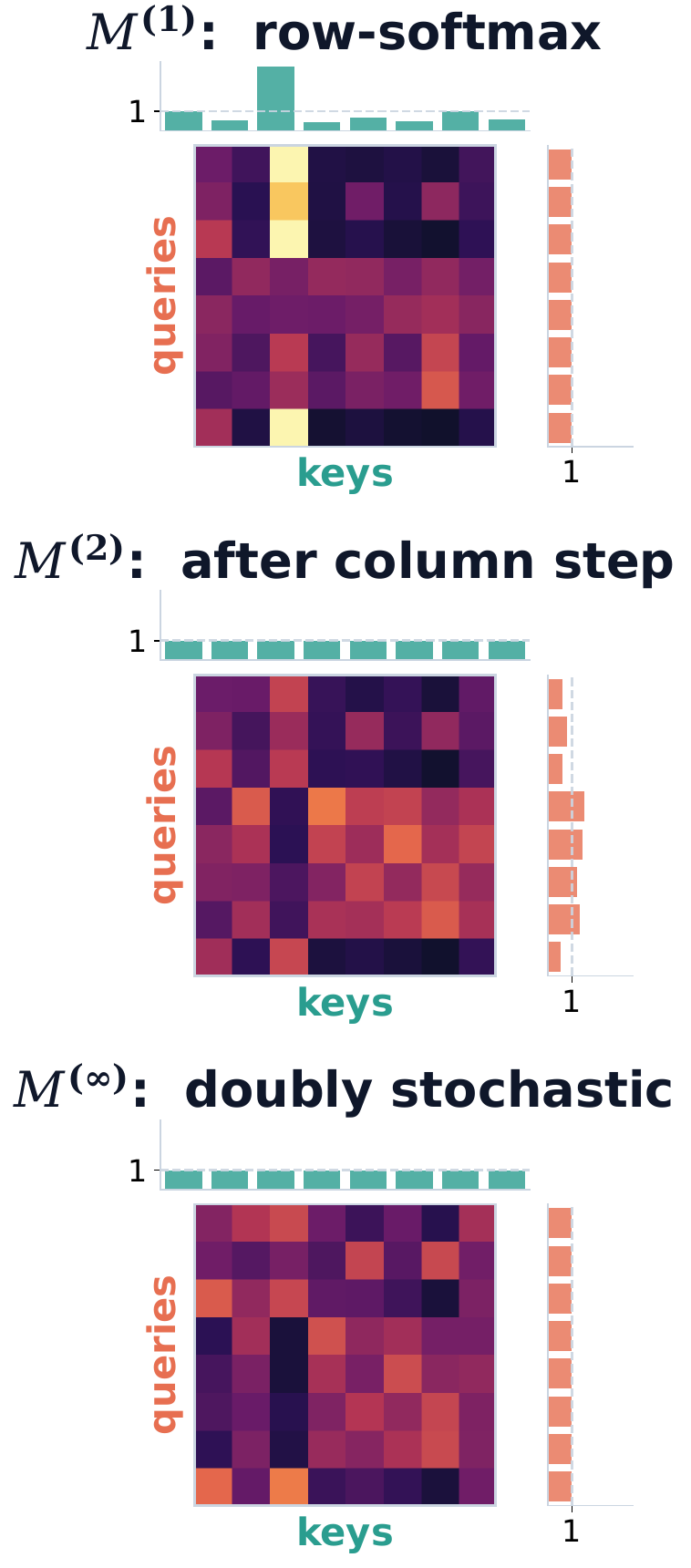}}
\vspace{-38pt}
\end{wrapfigure}

\textbf{From row-softmax to doubly-stochastic attention.}
For one attention head, let $Q,K\in\mathbb{R}^{N\times d_h}$ and $V\in\mathbb{R}^{N\times d_v}$ have rows $q_i,k_j,v_j$. Standard self-attention first computes scaled dot-product scores
\begin{equation}
\label{eq:scaled_dot_product_score}
s_{ij}=\frac{q_i^\top k_j}{\sqrt{d_h}} .
\end{equation}
It then applies a row-wise softmax, so each query distributes one unit of attention across the keys. This row-wise normalization does not constrain the column sums, so a few keys can receive much more total attention than others. After rescaling by $1/N$, the attention matrix can also be read as a transport plan, where $P_{ij}$ is the amount of mass sent from query $i$ to key $j$. Doubly-stochastic attention adds the matching column constraint, so each key receives the same total mass. Sinkhorn attention reaches this structure by successively normalizing the rows and columns of the positive attention kernel $M^{(0)}_{ij}=\exp(s_{ij}/\varepsilon)$. The experiments use $\varepsilon=1$ unless stated otherwise. For a positive matrix $M$, we can write
\begin{equation}
\label{eq:row_column_normalization}
\begin{aligned}
N_R(M)_{ij}
&=
\frac{M_{ij}}{\sum_{r=1}^N M_{ir}},
\qquad
N_C(M)_{ij}
&=
\frac{M_{ij}}{\sum_{r=1}^N M_{rj}} .
\end{aligned}
\end{equation}
for the row and column normalizations. The Sinkhorn iterations are
\begin{equation}
\label{eq:sinkhorn_iterations}
M^{(\ell+1)}
=
\begin{cases}
N_R(M^{(\ell)}), & \ell \text{ even},\\
N_C(M^{(\ell)}), & \ell \text{ odd}.
\end{cases}
\end{equation}
The first row normalization is exactly the row-softmax attention matrix. A column normalization then rebalances the keys, and further iterations alternate between the two. As the iterations continue, the matrix approaches a doubly-stochastic attention matrix whose rows and columns both sum to one. In practice, Sinkformers use a finite number $S$ of Sinkhorn iterations and return the partially balanced attention matrix $A^{(S)}=M^{(S)}$.


\textbf{Transport interpretation and dual potentials.}
The row and column normalizations above are also the standard matrix-scaling computation for entropic OT. To write this in OT notation, let $P=A/N$ be the transport-scale matrix with total mass one. Its row marginal is the mass sent by query tokens, its column marginal is the mass received by key tokens, and doubly-stochastic attention corresponds to $\mu=\nu=\mathbf{1}_N/N$. The feasible balanced plans are
\begin{equation}
\label{eq:transport_polytope}
\mathcal{U}(\mu,\nu)
=
\left\{
P\in\mathbb{R}_+^{N\times N}
\;\mid\;
P\mathbf{1}_N=\mu,\;
P^\top\mathbf{1}_N=\nu
\right\}.
\end{equation}

We can choose a cost that uses the same query--key geometry as scaled dot-product attention, that is,
\begin{equation}
\label{eq:quadratic_attention_cost}
C_{ij}
=
\frac{1}{2\sqrt{d_h}}\|q_i-k_j\|_2^2 .
\end{equation}
The square equals $-s_{ij}$ plus query-only and key-only terms, so the cost kernel is the dot-product score kernel up to one query-only factor and one key-only factor. At convergence, Sinkhorn scaling absorbs such row-only and column-only factors into its scaling vectors, so the balanced plan is the same. We can therefore implement Sinkhorn attention on the trained dot-product scores while interpreting the limiting operator as entropic transport under the quadratic query--key cost. For a finite Sinkhorn teacher, the output also depends on score scale, masking, initialization, update order, and stopping time, so ASAP extracts and reconstructs duals under the same finite-budget convention as the trained layer.

The corresponding entropic OT problem is
\begin{equation}
\label{eq:entropic_ot}
P^\star
=
\argmin_{P\in\mathcal{U}(\mu,\nu)}
\langle C,P\rangle
+
\varepsilon
\sum_{i,j}P_{ij}\bigl(\log P_{ij}-1\bigr).
\end{equation}
The first term encourages low-cost query--key transport, while the entropy term keeps the plan positive. The marginal constraints are enforced by one scalar per query and one scalar per key. We write these scalars as $f^\star$ and $g^\star$, and the optimal plan solution is
\begin{equation}
\label{eq:gibbs_plan}
P^\star_{ij}
=
\exp\!\left(\frac{-C_{ij}+f^\star_i+g^\star_j}{\varepsilon}\right).
\end{equation}
OT calls $f^\star$ and $g^\star$ Kantorovich dual potentials. Here $f^\star$ and $g^\star$ are potentials in cost units. An implementation that stores log scalings $u$ and $v$ has $f^\star=\varepsilon u$ and $g^\star=\varepsilon v$, since with $K_{ij}=\exp(-C_{ij}/\varepsilon)$ we have $P^\star=\operatorname{diag}(e^{f^\star/\varepsilon})K\operatorname{diag}(e^{g^\star/\varepsilon})$. Multiplying by $N$ returns the attention-scale matrix $A^\star=NP^\star$, whose rows sum to one. Sinkhorn provides a useful compression of the balanced plan. Instead of learning all pairwise attention weights, ASAP learns the teacher's query-side dual and uses the $c$-transform below to fill in the key side.

\textbf{Key dual from one source dual.}
Now fix the source dual $f$. For each key $j$, the only unknown left is the scalar key offset that makes column $j$ have mass $\nu_j$. Solving that one-column normalization equation yields the key-side entropic $c$-transform. In other words,
\begin{equation}
\label{eq:key_c_transform}
T_K^\nu(f)_j
=
\varepsilon\log\nu_j
-
\varepsilon\log
\sum_{i=1}^N
\exp\!\left(\frac{-C_{ij}+f_i}{\varepsilon}\right).
\end{equation}
If we set $g=T_K^\nu(f)$ in the exponential plan form, the column sum is exact. That is,
\begin{equation}
\label{eq:key_marginal_exact}
\sum_i
\exp\!\left(
\frac{-C_{ij}+f_i+T_K^\nu(f)_j}{\varepsilon}
\right)
=
\nu_j .
\end{equation}
Thus any predicted source dual determines a transport plan with the correct key marginal. In the column-ending convention used by the main teachers, $f^{(T)}$ denotes the source-side finite Sinkhorn potential immediately before the teacher's final key-side closure, and $g^{(T)}=T_K^\nu(f^{(T)})$. These are finite potentials from the trained layer, not infinite-limit OT solutions. For a trained finite-budget Sinkhorn layer, ASAP-LS learns to predict $f^{(T)}(Q,K)$. Applying the final key-side transform to the teacher source dual defines the one-sided reference plan used by ASAP-0 and the KL fit,
\begin{equation}
\label{eq:teacher_reference_plan}
P^{(T)}_{ij}(Q,K)
=
\exp\!\left(
\frac{-C_{ij}(Q,K)+f^{(T)}_i(Q,K)+g^{(T)}_j(Q,K)}{\varepsilon}
\right),
\qquad
A^{(T)}(Q,K)=NP^{(T)}(Q,K).
\end{equation}
The source dual is still defined only up to an additive constant. If we add a constant to the source dual, the key dual absorbs the opposite shift and the exponential plan stays unchanged, so we center $f^{(T)}$ before fitting. With perfect prediction, ASAP-0 recovers the finite teacher plan under the same final-side closure. Because the $c$-transform enforces the key marginal after every estimate, the source-dual estimate carries the remaining error, and the row profile and teacher-plan agreement depend on how closely $\hat f$ follows the teacher dual.

\subsection{One-Dimensional Sliced Potentials}
\label{sec:background_sliced_features}

\textbf{Why slicing.}
We need query-token features for the Sinkhorn source dual that can be computed without Sinkhorn iterations. A slice supplies such a feature by projecting queries and keys onto one line, where quadratic OT has a closed-form sorted matching. The sorted matching assigns one source-potential value to each query. Recent work has also used sliced OT quantities as reusable features for Wasserstein geometry, distance prediction, and amortized potentials \citep{haviv2024wassersteinwormhole,nguyen2026fastestimation,truong2026aot}. ASAP uses a fixed bank of these one-dimensional potentials as transport coordinates for the trained layer's finite Sinkhorn source dual.

\textbf{Sliced projections.}
For $\theta_\ell\in\mathbb{S}^{d_h-1}$, we use projections
\begin{equation}
\label{eq:sliced_projections}
a_i^{(\ell)}
=
\frac{\theta_\ell^\top q_i}{d_h^{1/4}},
\qquad
b_j^{(\ell)}
=
\frac{\theta_\ell^\top k_j}{d_h^{1/4}} .
\end{equation}
This scaling aligns the one-dimensional cost $\frac{1}{2}(a_i^{(\ell)}-b_j^{(\ell)})^2$ with the projection of the head-level cost $C$ along $\theta_\ell$, so the sliced potentials are in the same units as the entropic dual.

\textbf{One-dimensional source potential.}
Let $a_{(1)}^{(\ell)}\leq\cdots\leq a_{(N)}^{(\ell)}$ and $b_{(1)}^{(\ell)}\leq\cdots\leq b_{(N)}^{(\ell)}$ be the sorted projections. Uniform 1D OT with cost $\frac{1}{2}(a-b)^2$ matches equal ranks, so the $r$-th sorted query is paired with the $r$-th sorted key. The corresponding source potential is the discrete one-dimensional Kantorovich potential for this sorted matching. We can write it as
\begin{equation}
\label{eq:sliced_source_potential}
\begin{aligned}
\phi^{(\ell)}_1=0,
\qquad
\phi^{(\ell)}_r
=
\sum_{t=1}^{r-1}
b_{(t)}^{(\ell)}
\bigl(a_{(t+1)}^{(\ell)}-a_{(t)}^{(\ell)}\bigr),
\qquad
f^{(\ell)}_{(r)}
=
\frac{1}{2}\bigl(a_{(r)}^{(\ell)}\bigr)^2-\phi^{(\ell)}_r,
\qquad
r=1,\ldots,N .
\end{aligned}
\end{equation}
We then return these values to the original query order and center them to obtain $\tilde f^{(\ell)}(Q,K)\in\mathbb{R}^N$.

\textbf{Sliced feature matrix.}
For fixed directions $\Theta=\{\theta_\ell\}_{\ell=1}^L$, stack the centered slice potentials as $X_\Theta(Q,K)=[\tilde f^{(1)}(Q,K)\;\cdots\;\tilde f^{(L)}(Q,K)]\in\mathbb{R}^{N\times L}$. All $L$ slices cost $O(LNd_h+LN\log N)$, with the work split between projections and per-slice sorting. The rows of $X_\Theta$ are features, not attention rows, so ASAP does not average sliced plans into attention. It uses sliced source potentials as token-level coordinates and constructs the dense plan only after the source-dual estimate passes through the entropic $c$-transform.

\begin{figure}[t]
  \centering
  \includegraphics[width=0.96\textwidth]{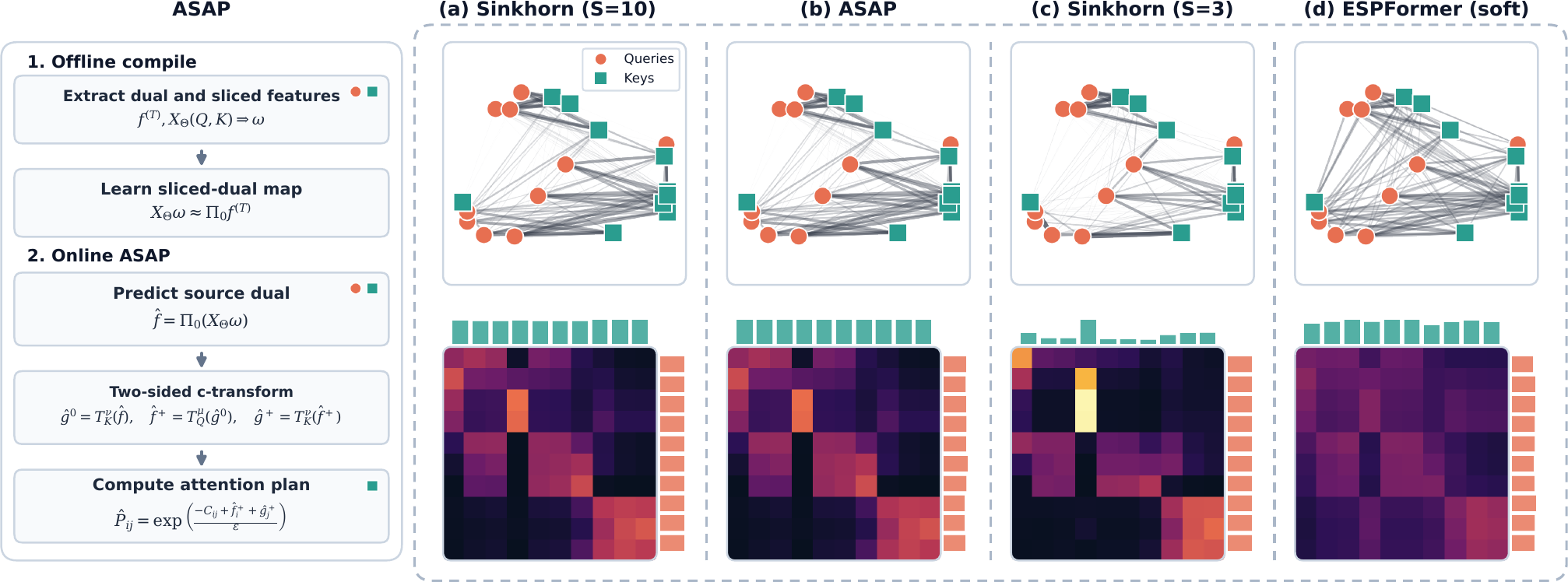}
  \caption{ASAP follows an offline-to-online compile process for a trained Sinkhorn attention layer. The right panels compare teacher fidelity and marginal balance on one query--key example. Here Sinkhorn ($S=10$) is the teacher. Flatter row and column bars indicate a more doubly-stochastic plan. ASAP stays close to the teacher, while Sinkhorn (low $S$) and ESPFormer (soft) are further from being doubly stochastic.}
  \label{fig:asap_method_overview}
\end{figure}

\section{Method}

\subsection{The ASAP Layer}
\label{sec:method_operator_parameterization}

\textbf{Offline compilation.}
After the Sinkhorn teacher is trained, ASAP freezes the attention layer and learns the source-dual computation that the Sinkhorn loop would otherwise repeat. Offline, we pass unlabeled calibration examples through the original finite Sinkhorn iterations and extract the closure-ready source dual produced under the teacher's update convention. On the same query and key activations, we compute the sliced potential matrix from Section~\ref{sec:background_sliced_features}. The offline step learns a coefficient vector that maps these sliced potentials to the extracted source dual.

At inference, the compiled layer follows the same steps without rerunning the Sinkhorn loop. It computes sliced potentials, predicts the source dual, applies the entropic $c$-transform, and constructs the compiled attention plan under the teacher's finite Sinkhorn convention. ASAP compiles one frozen Sinkhorn layer on the activation distribution it sees after training.

\textbf{Compiled layer.}
Formally, for one attention head, let $Q,K \in \mathbb{R}^{N\times d_h}$ be the query and key matrices of the frozen layer, and let $\Pi_0(z)=z-\frac{1}{N}(\mathbf{1}_N^\top z)\mathbf{1}_N$ choose the zero-mean representative. The compiled layer stores the slice bank $\Theta$ and coefficient vector $\omega$. For a new $Q,K$, it predicts a centered source dual from sliced potentials. With closure-ready teacher source dual $f^{(T)}(Q,K)$ and sliced feature matrix $X_\Theta(Q,K)$, ASAP uses
\begin{equation}
\label{eq:asap_source_prediction}
\begin{aligned}
\tilde f^{(T)}(Q,K)
&=\Pi_0\!\bigl(f^{(T)}(Q,K)\bigr),\\
\hat f_\omega(Q,K)
&=\Pi_0\!\bigl(X_\Theta(Q,K)\omega\bigr),
\qquad \omega\in\mathbb{R}^L .
\end{aligned}
\end{equation}

\textbf{Two-sided entropic $c$-transform.}
Once we choose the source and key duals, the attention plan is fixed and can be written as
\begin{equation}
\label{eq:plan_from_duals}
P_C(f,g)_{ij}
=
\exp\!\left(\frac{-C_{ij}+f_i+g_j}{\varepsilon}\right).
\end{equation}
Equation~\ref{eq:key_c_transform} already tells us how to take any source dual and choose the key dual so that the columns have mass $\nu$. If the source dual is only an estimate, this column step is still exact, but the rows carry the remaining source-dual error. ASAP therefore applies the same normalization idea on the source side. Fix a key dual $g$. For each query $i$, the only unknown left is the scalar source offset that makes row $i$ have mass $\mu_i$. Solving that one-row normalization equation yields
\begin{equation}
\label{eq:source_c_transform}
T_Q^\mu(g)_i
=
\varepsilon\log\mu_i
-
\varepsilon\log
\sum_{j=1}^N
\exp\!\left(\frac{-C_{ij}+g_j}{\varepsilon}\right),
\end{equation}
the source-side mirror of the key-side transform.

Default ASAP starts from the predicted source dual $\hat f_\omega$. First it closes the key side and sets $\hat g^0_\omega=T_K^\nu(\hat f_\omega)$. Then it closes the source side and sets $\hat f^+_\omega=T_Q^\mu(\hat g^0_\omega)$. This pair has the correct row marginal, but the key dual was computed against the old source dual $\hat f_\omega$, so ASAP returns to the key side once more. In the column-ending case, the full sequence is
\begin{equation}
\label{eq:two_sided_c_transform}
\hat g^0_\omega
\!=\!
T_K^\nu(\hat f_\omega),
\qquad
\hat f^+_\omega
\!=\!
T_Q^\mu(\hat g^0_\omega),
\qquad
\hat g^+_\omega
 \!=\!
T_K^\nu(\hat f^+_\omega).
\end{equation}
We call this the two-sided entropic $c$-transform and write $T_{\mu,\nu}^{\leftrightarrow}(\hat f_\omega;C)=(\hat f^+_\omega,\hat g^+_\omega)$. The compiled attention plan is then
\begin{equation}
\label{eq:asap_reconstruction}
\hat P_\omega
=
P_C(\hat f^+_\omega,\hat g^+_\omega),
\qquad
\hat A_\omega=N\hat P_\omega .
\end{equation}
With exact prediction of the closure-ready teacher source dual, ASAP-0 recovers the finite teacher under the same final-side convention. The two-sided ASAP operator applies further finite closures from that same starting potential and need not equal the finite teacher away from a fixed point. If a teacher ends with row normalization, ASAP uses the analogous row-ending form. Appendix~\ref{app:details} contains the general algorithm.

We also report ASAP-0 as a one-sided speed ablation. It stops after the first key-side transform,
\begin{equation}
\label{eq:asap_zero_operator}
\hat g^0_\omega
=
T_K^\nu(\hat f_\omega),
\qquad
\hat P^0_\omega
=
P_C(\hat f_\omega,\hat g^0_\omega).
\end{equation}
ASAP-0 reports the learned source-dual projection without the extra log-sum-exp passes used by default ASAP. In all cases, agreement with the finite teacher depends on the predicted source dual. The head output is $\hat y_i=\sum_j \hat A_{\omega,ij}v_j$, following the convention $A=NP$ used by the Sinkhorn teacher.

\subsection{Sliced-Dual Projection}
\label{sec:method_training_complexity}

Once the online layer is fixed, calibration only has to choose the coefficient vector $\omega$. We call this calibration step \textbf{sliced-dual projection}, since it projects the teacher's closure-ready finite source dual onto the span of sliced-potential features for that frozen attention layer.

\noindent\textbf{ASAP-LS.}
The closed-form calibration fits the teacher source dual in centered coordinates. Each calibration sequence contributes $N$ token-level regression rows, with row $i$ of $X_\Theta$ holding the $L$ sliced-potential values for query token $i$ and target $\tilde f^{(T)}_i$. For calibration examples $(Q^{(m)}, K^{(m)}) \in \mathcal{D}$, it solves
\begin{equation}
\label{eq:asap_ls}
\omega^{\mathrm{LS}}
\in
\arg\min_{\omega \in \mathbb{R}^L}
\sum_{m=1}^M
\left\|
X_\Theta(Q^{(m)}, K^{(m)})\omega
- \tilde f^{(T)}(Q^{(m)}, K^{(m)})
\right\|_2^2
+
\lambda \|\omega\|_2^2,
\end{equation}
where $\lambda > 0$ is a ridge parameter. If $\bar X$ denotes the stacked feature rows and $\bar y$ the stacked target entries, the closed-form estimator is
\begin{equation}
\label{eq:asap_ls_closed_form}
\omega^{\mathrm{LS}}
=
(\bar X^\top \bar X + \lambda I_L)^{-1}\bar X^\top \bar y.
\end{equation}
For a deep backbone, we calibrate each attention layer on its own. ASAP-LS is the default in the main experiments because it is closed form and the choice of $\omega$ does not change the online layer. The normalized risk statements in Appendix~\ref{app:proofs} use $\lambda_{\mathrm{risk}}$. Under that normalization, the implementation ridge in Equation~\ref{eq:asap_ls} is $\lambda=MN\lambda_{\mathrm{risk}}$.

\noindent\textbf{ASAP-KL.}
ASAP-KL keeps the same source-dual parameterization but chooses $\omega$ through a convex one-sided teacher-to-model plan KL objective. The fit passes through the key-normalized plan formed after the first key-side $c$-transform, so it is a calibration variant of the same online layer. ASAP-KL changes only the offline objective, and the fitted coefficients are evaluated with the same ASAP or ASAP-0 transform as the corresponding LS row. Appendix~\ref{app:proofs} contains the objective and convexity proof.

Relative to Sinkhorn attention, ASAP replaces only the iteration factor in online normalization. ASAP still uses the same dense score and value operations, so the savings come from replacing the $S$ Sinkhorn iterations with sliced-potential extraction, a length-$L$ source-dual prediction, and the entropic $c$-transform calls in Equation~\ref{eq:two_sided_c_transform}. Teacher extraction and fitting are offline costs for a fixed trained layer. Appendix~\ref{app:cost_breakdown} reports the cost details and break-even points.

\section{Experiments}
\label{sec:experiments}

The experiments follow the train-then-compile question from Sections~\ref{sec:background_transport_attention}--\ref{sec:method_training_complexity}. We first freeze a trained Sinkhorn layer and ask whether ASAP preserves its finite operator under the same query, key, value, and output projections. We then replace trained Sinkhorn teachers in downstream vision and text models, and compare with ESPFormer while treating its hard-sort inference operator and soft-sort training operator as distinct operators.

We report trained attention families and post-training replacements in their own rows. Transformer, DiffTransformer, Sinkformer, and ESPFormer rows are trained models. A Sinkhorn teacher is the high-budget trained model or frozen layer whose finite operator is compiled by ASAP. A Sinkhorn normalizer reuses a teacher's frozen projections but reruns a specified number of Sinkhorn updates. ASAP rows are post-training replacements fit from unlabeled calibration inputs, with labels used only to train and evaluate the original models. Unless stated otherwise, the teachers in the main tables end with a column normalization, so ASAP uses the column-ending two-sided transform in Equation~\ref{eq:two_sided_c_transform}. ASAP-0 denotes the one-sided ablation from Equation~\ref{eq:asap_zero_operator}, and ASAP-KL tests the same online layer with the KL fit.

ESPFormer uses Expected Sliced Transport Plans as its attention construction. The resource gap comes from where the dense tensors appear during training. Sinkhorn repeatedly normalizes one dense $B\times H\times N\times N$ kernel, while ESPFormer's soft-sort training builds query and key sorting relaxations with an extra slice axis, $B\times H\times L\times N\times N$, before averaging slice plans. The left panel of Table~\ref{tab:operator_main} reports this full-encoder IMDb resource profile at sequence length $512$ and effective batch size $32$, with Sinkhorn budgets $S\in\{3,10,20\}$. Against the $S=10$ teacher budget used in the downstream text experiments, this soft-sort training operator uses $5.7\times$ more peak training memory and runs $14.8\times$ slower per training example. The right panel then asks whether ASAP can keep the trained Sinkhorn teacher operator while reducing online scaling.

\subsection{Frozen-Layer Replacement}
\label{sec:experiments_operator}

We first test the compiler without end-to-end training. On real IMDb activations, all rows use the same query, key, value, and output projections from an $S=20$ Sinkhorn teacher. The Sinkhorn normalizer rows lower the teacher budget after training, while ASAP rows substitute compiled operators fit from that teacher.

Table~\ref{tab:operator_main} compares two ways to make a trained Sinkhorn layer cheaper. Reducing the Sinkhorn budget is the obvious baseline, but it changes the finite normalizer. Even the $S=3$ replacement is slower than ASAP-LS-0 and is much farther from the $S=20$ teacher in both output and attention error. ASAP instead keeps the trained layer fixed and learns the source-dual computation that the loop would otherwise repeat. The one-sided ASAP-LS-0 row shows the raw speed of this amortization. The default two-sided ASAP-LS row is slightly slower than ASAP-LS-0, but it is much closer to the teacher. Output RMSE falls to $0.003$, attention relative $\ell_2$ falls to $0.035$, and the frozen classifier matches the teacher accuracy. The useful comparison is therefore not only latency, but whether the cheaper layer still behaves like the trained Sinkhorn operator.

The frozen-layer results compare latency with fidelity. The $S=3$ and $S=5$ normalizers are faster than the $S=20$ teacher because they stop the same scaling loop earlier, but they execute another finite operator. ASAP-0 shows the raw learned-dual projection speed by computing only the key-side transform. The default ASAP operator applies the two-sided transform, spends a small fixed number of log-sum-exp passes, and stays much closer to the teacher without rerunning the full scaling loop.

\begin{table}[!htbp]
  \centering
  \begin{minipage}[t]{0.35\textwidth}
  \vspace{0pt}
  \centering
  \raisebox{2.7ex}{\includegraphics[width=\linewidth,trim=0 0 0 0,clip]{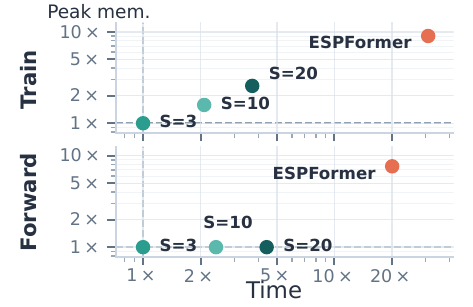}}
  \end{minipage}
  \hfill
  \begin{minipage}[t]{0.62\textwidth}
  \vspace{2.7ex}
  \centering
  \scriptsize
  \setlength{\tabcolsep}{3pt}
  \renewcommand{\arraystretch}{0.93}
  \resizebox{\linewidth}{!}{%
  \begin{tabular}{lcccc}
    \toprule
    Method & Time (ms) $\downarrow$ & Output RMSE $\downarrow$ & Attention rel.\ $\ell_2$ $\downarrow$ & IMDb acc.\ (\%) $\uparrow$ \\
    \midrule
    Sinkhorn normalizer ($S=3$) & $14.14 \pm 0.03$ & $0.068 \pm 0.030$ & $1.527 \pm 0.319$ & $86.92 \pm 0.12$ \\
    Sinkhorn normalizer ($S=5$) & $19.87 \pm 0.05$ & $0.060 \pm 0.015$ & $1.430 \pm 0.241$ & $87.15 \pm 0.14$ \\
    \midrule
    ASAP-LS-0 ($L=32$) & $\mathbf{7.74 \pm 0.02}$ & $0.020 \pm 0.017$ & $0.364 \pm 0.255$ & $87.29 \pm 0.20$ \\
    ASAP-LS-0 ($L=64$) & $9.19 \pm 0.02$ & $0.018 \pm 0.012$ & $0.322 \pm 0.177$ & $87.25 \pm 0.20$ \\
    ASAP-KL-0 ($L=32$) & $\mathbf{7.74 \pm 0.02}$ & $0.020 \pm 0.016$ & $0.365 \pm 0.239$ & $87.29 \pm 0.24$ \\
    ASAP-LS ($L=32$) & $11.79 \pm 0.02$ & $\mathbf{0.003 \pm 0.004}$ & $\mathbf{0.035 \pm 0.070}$ & $\mathbf{87.36 \pm 0.21}$ \\
    ASAP-KL ($L=32$) & $11.79 \pm 0.02$ & $0.003 \pm 0.005$ & $0.039 \pm 0.072$ & $87.34 \pm 0.24$ \\
    \midrule
    Sinkhorn teacher ($S=20$) & $63.02 \pm 0.12$ & -- & -- & $87.34 \pm 0.27$ \\
    \bottomrule
  \end{tabular}
  }
  \end{minipage}
  \caption{IMDb resource profile and frozen-normalizer replacement. The left panel reports train-step time and memory at sequence length $512$, normalized to Sinkformer $S=3$. The right panel reports teacher-relative operator error and frozen-classifier accuracy after replacing the normalizer in the same $S=20$ teacher.}
  \label{tab:operator_main}
\end{table}
\FloatBarrier

\subsection{Vision Benchmarks}
\label{sec:experiments_catsdogs}

We next test downstream preservation on vision tasks. Cats and Dogs follows the full-data ViT protocol reported by ESPFormer \citep{shahbazi2025espformer}. We keep the published ESPFormer rows under that protocol and train our Sinkhorn teachers under the same full-data ViT setting.

\begin{table}[!htbp]
  \centering
  \scriptsize
  \setlength{\tabcolsep}{3pt}
  \begin{tabular*}{0.86\linewidth}{@{\extracolsep{\fill}}lccc@{}}
    \toprule
    Method & Acc.\ (\%) $\uparrow$ & Train peak (GB) $\downarrow$ & Compile (s) $\downarrow$ \\
    \midrule
    Transformer & $78.49 \pm 0.09^\dagger$ & $1.64$ & -- \\
    DiffTransformer & $78.85 \pm 0.11^\dagger$ & $4.22$ & -- \\
    Sinkformer ($S=3$) & $79.12 \pm 0.17^\dagger$ & $3.54$ & -- \\
    ESPFormer (Soft, init) & $79.47 \pm 0.12^\dagger$ & $170.45$ & -- \\
    ESPFormer (Soft, sharp) & $80.61 \pm 0.11^\dagger$ & $170.45$ & -- \\
    ESPFormer (Hard) & $81.23 \pm 0.11^\dagger$ & $170.45$ & -- \\
    \midrule
    Sinkhorn teacher ($S=10$) & $81.09 \pm 0.61$ & $6.90$ & -- \\
    ASAP-LS-0 ($L=32$) & $80.15 \pm 0.69$ & $6.90$ & $24.38 \pm 0.16$ \\
    ASAP-KL-0 ($L=32$) & $79.77 \pm 0.90$ & $6.90$ & $101.19 \pm 0.76$ \\
    ASAP-LS ($L=64$) & $80.44 \pm 0.39$ & $6.90$ & $25.02 \pm 0.17$ \\
    \bottomrule
  \end{tabular*}
  \caption{Cats and Dogs full-data ViT results. ESPFormer rows marked $^\dagger$ are published results. ASAP rows compile three Sinkhorn teachers ($S=10$), and resource cells use released ESPFormer implementations.}
  \label{tab:catsdogs_main}
\end{table}

ASAP-LS-0 keeps the Sinkhorn teacher within $0.94$ points after a $24.38$ second compile step. The default ASAP-LS operator reduces the gap to $0.65$ points with a similar compile cost. ESPFormer (Hard) remains strong, but it is a hard-sort inference operator reached after soft-sort training, and the switch may not improve uniformly across tasks. Appendix~\ref{app:esp_soft_hard_switch} studies the same soft-to-hard switch on CIFAR-100 checkpoints, where exact hard sorting restores marginal balance but changes $34.65\%$ of test predictions from the soft operator. ASAP answers another question of how much of that finite operator survives after a Sinkhorn teacher has been trained and the scaling loop is replaced.

\begin{table}[!htbp]
  \centering
  \scriptsize
  \sbox{\cifarTablePanel}{%
    \begin{minipage}[t]{0.43\linewidth}
      \vspace{0pt}
      \setlength{\tabcolsep}{3pt}
      \begin{tabular}{@{}p{0.62\linewidth}r@{}}
        \toprule
        Method & Acc.\ (\%) $\uparrow$ \\
        \midrule
        Transformer & $62.73 \pm 0.01$ \\
        DiffTransformer & $60.03 \pm 0.59$ \\
        ESPFormer & $58.46 \pm 0.42$ \\
        Sinkhorn teacher ($S=10$) & $\mathbf{64.95 \pm 0.06}$ \\
        \midrule
        ASAP-LS-0 ($L=32$) & $63.64 \pm 0.42$ \\
        ASAP-KL-0 ($L=32$) & $63.58 \pm 0.66$ \\
        ASAP-LS ($L=32$) & $63.46 \pm 0.36$ \\
        ASAP-KL ($L=32$) & $63.34 \pm 0.44$ \\
        \bottomrule
      \end{tabular}
    \end{minipage}%
  }%
  \usebox{\cifarTablePanel}\hfill
  \begin{minipage}[t]{0.54\linewidth}
    \vspace{0pt}
    \includegraphics[width=\linewidth,height=\dimexpr\ht\cifarTablePanel+\dp\cifarTablePanel\relax]{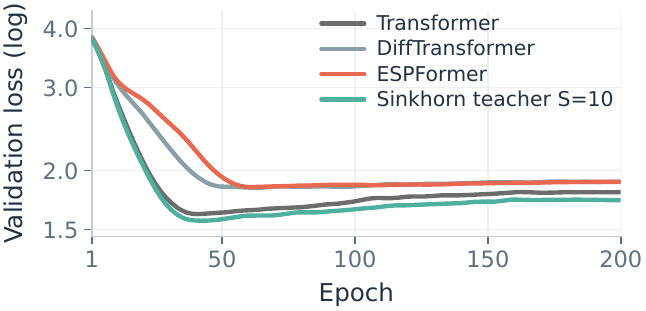}
  \end{minipage}
  \caption{CIFAR-100 patch-size-4 control. The left panel reports final accuracy for trained baselines and ASAP replacements. The right panel reports validation loss for trained baselines only.}
  \label{tab:cifar100_available}
\end{table}

CIFAR-100 is our harder vision control because preserving the teacher requires more than the present linear sliced-dual map provides. The Sinkhorn teacher reaches the best final accuracy, and its validation loss falls below row-softmax, DiffTransformer, and ESPFormer early in training. This matches the Sinkformer view of column balance as an informative prior during optimization. ASAP remains above the row-softmax and ESPFormer baselines without retraining, but the remaining gap shows where the current linear sliced-dual map loses teacher fidelity.

\subsection{Text Classification}
\label{sec:experiments_text_classification}

Text classification is the cleanest downstream preservation study. DBpedia-14, AG News \citep{zhang2015character}, and Yelp Review Full use the same encoder-plus-pooling family, an $S=10$ Sinkhorn teacher, and ASAP compiled from $4096$ unlabeled training examples. The datasets span highly distinct topics, four-way news, and five-class sentiment, with the same backbone and calibration procedure throughout.

Across these datasets, ASAP-0 trails the $S=10$ teacher by less than one point, and the default ASAP operator recovers the teacher to within about $0.1$ points on every dataset. ESPFormer trails the Sinkhorn teacher in this protocol, most sharply on Yelp Review Full. Because ASAP-0 and ASAP use the same compiled source-dual prediction, this setting directly tests the value of applying the entropic $c$-transform on both marginal sides.

\begin{table}[!htbp]
  \centering
  \scriptsize
  \setlength{\tabcolsep}{3.5pt}
  \begin{tabular}{lccc}
    \toprule
    Method & DBpedia-14 & AG News & Yelp Review Full \\
    \midrule
    Transformer & $96.99 \pm 0.42$ & $89.34 \pm 0.20$ & $54.40 \pm 0.05$ \\
    DiffTransformer & $96.80 \pm 0.33$ & $88.01 \pm 0.15$ & $54.30 \pm 0.06$ \\
    ESPFormer & $96.58 \pm 0.01$ & $86.50 \pm 0.43$ & $49.49 \pm 0.38$ \\
    Sinkhorn teacher ($S=10$) & $\mathbf{98.27 \pm 0.02}$ & $\mathbf{90.12 \pm 0.12}$ & $\mathbf{54.54 \pm 0.25}$ \\
    \midrule
    ASAP-LS-0 & $97.82 \pm 0.05$ & $89.61 \pm 0.18$ & $53.87 \pm 0.47$ \\
    ASAP-KL-0 & $97.90 \pm 0.06$ & $89.59 \pm 0.20$ & $53.83 \pm 0.44$ \\
    ASAP-LS & $98.26 \pm 0.03$ & $90.03 \pm 0.17$ & $54.51 \pm 0.35$ \\
    ASAP-KL & $98.25 \pm 0.04$ & $90.04 \pm 0.18$ & $54.46 \pm 0.37$ \\
    \bottomrule
  \end{tabular}
  \caption{Text classification accuracy (\%) at sequence length $128$. ASAP rows compile $S=10$ Sinkhorn teachers from $4096$ unlabeled examples. Entries are mean $\pm$ standard deviation over three runs.}
  \label{tab:text_classification_main}
\end{table}

These results address the two questions that guide the paper, whether Sinkhorn attention is useful during training and whether ASAP preserves the trained operator after replacement. Sinkhorn attention improves the teacher rows most consistently on CIFAR-100 and text, whereas ASAP measures how much of that learned computation remains after we replace most of the repeated scaling loop. Sequence classification transfers tightly with the default two-sided transform, while CIFAR-100 is the hardest setting for the present source-dual feature map.

\FloatBarrier
\vspace{-10pt}

\section{Limitations}
\vspace{-5pt}

The ASAP layer replaces the iterative Sinkhorn scaling loop, but it still uses dense query--key scores and dense value mixing. ASAP therefore reduces online scaling work without making attention sparse or subquadratic. It also assumes a repeated-inference setting where teacher extraction and offline calibration can be amortized.

Additionally, a linear sliced-dual map may not be expressive enough for all query--key geometries. Large shifts in the layer's activation distribution may require recalibration. Causal attention also needs its own marginal design, since a strict triangular mask is not compatible with the square uniform two-marginal setting used in the main text.

\vspace{-8pt}
\section{Conclusion}
\vspace{-5pt}

We introduced ASAP as an approach to keep the Sinkhorn training setup while avoiding most of the repeated scaling work in a trained attention layer. The method learns a lightweight map from sliced one-dimensional OT potentials to the teacher's finite source dual and applies a two-sided entropic $c$-transform under the teacher's Sinkhorn convention. Across frozen-layer and downstream replacements, ASAP follows the trained teacher more closely than lowering the Sinkhorn budget while requiring no retraining of the surrounding model. Future work could extend the same compiler idea to Sinkhorn-based sparse attention, where Sinkhorn balances block-level routing or sorting matrices before local attention.

\bibliographystyle{plainnat}
\bibliography{references}

\clearpage
\appendix
\setcounter{figure}{0}
\setcounter{table}{0}
\setcounter{algorithm}{0}
\setcounter{equation}{0}
\renewcommand{\thefigure}{\Alph{section}.\arabic{figure}}
\renewcommand{\thetable}{\Alph{section}.\arabic{table}}
\renewcommand{\thealgorithm}{\Alph{section}.\arabic{algorithm}}
\renewcommand{\theequation}{\Alph{section}.\arabic{equation}}
\makeatletter
\@addtoreset{figure}{section}
\@addtoreset{table}{section}
\@addtoreset{algorithm}{section}
\@addtoreset{equation}{section}
\makeatother
\renewcommand{\theHfigure}{appendix.\thesection.\arabic{figure}}
\renewcommand{\theHtable}{appendix.\thesection.\arabic{table}}
\renewcommand{\theHalgorithm}{appendix.\thesection.\arabic{algorithm}}
\renewcommand{\theHequation}{appendix.\thesection.\arabic{equation}}

\section{Notation and Conventions}
\label{app:notation}
\begingroup
\small
\setlength{\tabcolsep}{5pt}
\renewcommand{\arraystretch}{1.12}
\begin{longtable}{@{}p{0.25\linewidth}p{0.67\linewidth}@{}}
\toprule
Symbol & Convention \\
\midrule
\endfirsthead
\toprule
Symbol & Convention \\
\midrule
\endhead
\midrule
\multicolumn{2}{r}{\emph{continued on next page}}\\
\bottomrule
\endfoot
\bottomrule
\endlastfoot
\multicolumn{2}{@{}l}{\textbf{Basic objects.}}\\
\addlinespace[2pt]
$\mathbf{1}_m$ & all-ones vector in $\mathbb{R}^m$ \\
$\mathrm{Diag}(a)$ & diagonal matrix with diagonal entries $a_i$ \\
$\Delta_m$ & $\{p\in\mathbb{R}^m_+:\mathbf{1}_m^\top p=1\}$ \\
$\Delta_m^+$ & $\{p\in\Delta_m:p_i>0\ \forall i\}$ \\
$\langle C,P\rangle$ & $\sum_{ij}C_{ij}P_{ij}$ for matrices of the same size \\
$\Pi_0(z)$ & $z-\frac{1}{N}(\mathbf{1}_N^\top z)\mathbf{1}_N$, the zero-mean projection \\
\addlinespace[6pt]
\multicolumn{2}{@{}l}{\textbf{Tokens and attention.}}\\
\addlinespace[2pt]
$N$ & sequence length after padding or patch/token construction \\
$d_h,d_v$ & head dimension and value dimension \\
$Q,K,V$ & matrices in $\mathbb{R}^{N\times d_h}$, $\mathbb{R}^{N\times d_h}$, and $\mathbb{R}^{N\times d_v}$ \\
$q_i,k_j,v_j$ & rows of $Q,K,V$ \\
$S_{ij}$ & scaled dot-product score $\langle q_i,k_j\rangle/\sqrt{d_h}$ \\
$C_{ij}$ & quadratic cost $\|q_i-k_j\|_2^2/(2\sqrt{d_h})$ used in the main text \\
$\rho_i(Q),\kappa_j(K)$ & score-to-cost shift terms $\|q_i\|_2^2/(2\sqrt{d_h})$ and $\|k_j\|_2^2/(2\sqrt{d_h})$ \\
\addlinespace[6pt]
\multicolumn{2}{@{}l}{\textbf{Transport and Sinkhorn quantities.}}\\
\addlinespace[2pt]
$\mu,\nu$ & uniform marginals $\frac{1}{N}\mathbf{1}_N$ in the unmasked square setting \\
$\mathcal{U}(\mu,\nu)$ & $\{P\in\mathbb{R}^{N\times N}_+:P\mathbf{1}_N=\mu,\ P^\top\mathbf{1}_N=\nu\}$ \\
$P$ & coupling in transport units \\
$A$ & attention-scale matrix $A=NP$ used for value mixing \\
$f,g$ & source and key Sinkhorn duals, defined up to additive shifts $(f,g)\mapsto(f+c\mathbf{1}_N,g-c\mathbf{1}_N)$ \\
$\varepsilon$ & entropic regularization parameter or its equivalent score-kernel scale \\
$T_K^\nu(f)$ & key-side entropic $c$-transform applied to a source dual $f$ \\
$T_Q^\mu(g)$ & source-side entropic $c$-transform applied to a key dual $g$ \\
$T$ & finite Sinkhorn budget used by the teacher \\
$S$ & finite Sinkhorn budget used by an iterative comparison baseline \\
\addlinespace[6pt]
\multicolumn{2}{@{}l}{\textbf{ASAP fitting quantities.}}\\
\addlinespace[2pt]
$\Theta=\{\theta_\ell\}_{\ell=1}^L$ & fixed slice directions used by an ASAP layer \\
$a_i^{(\ell)},b_j^{(\ell)}$ & one-dimensional projections of scaled queries and keys along $\theta_\ell$ \\
$\tilde f^{(\ell)}$ & centered one-dimensional source potential from slice $\ell$ \\
$X_\Theta(Q,K)$ & matrix in $\mathbb{R}^{N\times L}$ whose columns are $\tilde f^{(\ell)}(Q,K)$ \\
$M$ & number of unlabeled fit examples used by an offline projection \\
$\lambda$ & regularization parameter for an offline projection \\
$\omega$ & ASAP coefficient vector in $\mathbb{R}^L$ \\
$\hat f_\omega$ & centered predicted source dual $\Pi_0(X_\Theta(Q,K)\omega)$ \\
{\bf ASAP-KL} & ASAP with the convex KL fit through the key-side $c$-transform \\
{\bf ASAP-LS} & ASAP with the closed-form least-squares dual fit \\
ASAP-0 & one-sided ablation using $\hat g^0_\omega=T_K^\nu(\hat f_\omega)$ \\
ASAP & default compiled layer using the two-sided transform in Equation~\ref{eq:two_sided_c_transform} \\
\addlinespace[6pt]
\multicolumn{2}{@{}l}{\textbf{Masks, errors, and timing.}}\\
\addlinespace[2pt]
$J(x)$ & active key positions for a padded text input $x$ \\
$\nu^J$ & masked key target with $\nu^J_j=1/N$ for $j\in J(x)$ and $\nu^J_j=0$ otherwise \\
$e_{\mathrm{row}}(A)$ & $N^{-1}\sum_i|\sum_j A_{ij}-1|$ in attention units \\
$e_{\mathrm{col}}(A;J)$ & $|J|^{-1}\sum_{j\in J}|\sum_i A_{ij}-1|$ for text batches, or $N^{-1}\sum_j|\sum_i A_{ij}-1|$ without padding \\
layer-batch case & one held-out minibatch evaluated at one frozen attention layer, with heads included in the average unless stated otherwise \\
fit time & teacher extraction and offline fitting; reported outside online forward latency \\
\end{longtable}
\endgroup

\section{Proofs}
\label{app:proofs}

\medskip\noindent\textbf{Exact row/column-shift invariance.}

\begin{proposition}
\label{prop:shift_invariance}
Fix marginals $\mu,\nu$ and $\varepsilon > 0$. If $\widetilde C_{ij} = C_{ij} + a_i + b_j$ for some $a,b \in \mathbb{R}^N$, then $C$ and $\widetilde C$ induce the same entropic optimal transport plan on $\mathcal{U}(\mu,\nu)$.
\end{proposition}

\begin{proof}
Let $J_C(P)=\langle C, P \rangle+\varepsilon\sum_{i,j}P_{ij}(\log P_{ij}-1)$ denote the entropic OT objective over $P\in\mathcal{U}(\mu,\nu)$. Suppose $P$ is feasible. Fixed marginals make row-only and column-only additions constant across the feasible set. Indeed, if $\widetilde C_{ij}=C_{ij}+a_i+b_j$, we have
\begin{equation}
J_{\widetilde C}(P)
=
J_C(P)
+
\sum_{i,j}(a_i + b_j)P_{ij}
=
J_C(P) + a^\top \mu + b^\top \nu.
\end{equation}
The added term $a^\top \mu+b^\top\nu$ does not depend on $P$. The two objectives are therefore the same up to a constant on the transport polytope, so they choose the same minimizers. For the attention kernel, we can take
\begin{equation}
\widetilde C_{ij} = -S_{ij},
\qquad
C_{ij} = \frac{1}{2\sqrt{d_h}}\|q_i-k_j\|_2^2,
\end{equation}
so that
\begin{equation}
-C_{ij}
=
S_{ij}
- \frac{1}{2\sqrt{d_h}}\|q_i\|_2^2
- \frac{1}{2\sqrt{d_h}}\|k_j\|_2^2
.
\end{equation}
So the quadratic-cost and scaled-dot-product formulations are the same up to row and column terms. For the exact two-marginal entropic problem, those terms are invisible to the optimizer, and the induced plan is the same.
\end{proof}

For exact entropic OT, the minimizer is unchanged by row and column shifts of the cost. With finitely many Sinkhorn updates, the output is also tied to the chosen kernel, initialization, and update order. ASAP therefore fits and reconstructs under the same finite-budget convention.

\medskip\noindent\textbf{Validity of the one-dimensional sliced potential.}

\begin{proposition}
\label{prop:oned_potential}
Let $a_{(1)} \leq \cdots \leq a_{(N)}$ and $b_{(1)} \leq \cdots \leq b_{(N)}$ be sorted one-dimensional supports with uniform masses. Define
\begin{equation}
\phi_1 = 0,
\qquad
\phi_r
=
\sum_{t=1}^{r-1} b_{(t)}(a_{(t+1)}-a_{(t)}),
\qquad r \geq 2,
\end{equation}
and $f_{(r)} = \frac{1}{2}a_{(r)}^2-\phi_r$. The vector $f$ is a valid source Kantorovich potential for the monotone optimal matching under the cost $c(a,b)=\frac{1}{2}(a-b)^2$.
\end{proposition}

\begin{proof}
Assume first that the source locations are all distinct. We want a dual pair that is tight exactly on the sorted matching. Extend $\phi$ to a piecewise-linear function with $\phi(a_{(r)})=\phi_r$ and slope $b_{(r)}$ on each interval $[a_{(r)},a_{(r+1)}]$, using outer rays with slopes $b_{(1)}$ and $b_{(N)}$. The sorted target values are nondecreasing, so these slopes are nondecreasing and $\phi$ is convex. We write $\phi^\star$ for its convex conjugate and define
\begin{equation}
h_{(s)}
=
\frac{1}{2}b_{(s)}^2-\phi^\star(b_{(s)}).
\end{equation}
By Fenchel's inequality, $\phi(a)+\phi^\star(b)\geq ab$ for all $a,b$. Therefore, for every source index $r$ and target index $s$, we have
\begin{equation}
f_{(r)} + h_{(s)}
\leq
\frac{1}{2}a_{(r)}^2+\frac{1}{2}b_{(s)}^2-a_{(r)}b_{(s)}
=
\frac{1}{2}(a_{(r)}-b_{(s)})^2.
\end{equation}
This is the dual-feasibility inequality for the quadratic-cost transport problem. On a matched pair $(a_{(r)},b_{(r)})$, our slope convention places $b_{(r)}$ in $\partial \phi(a_{(r)})$, so Fenchel equality holds and the dual bound is tight. The sorted matching is primal feasible. After we average these equalities over the uniform atoms, the primal cost and dual value agree, so the sorted matching is optimal and $f$ is a valid source potential.

If several source locations are tied, the zero gaps leave $\phi_r$ unchanged across the tied block. The sorted target values assigned to that block lie in the corresponding subgradient interval, so the same equality argument applies after fixing any deterministic monotone tie rule.
\end{proof}

\medskip\noindent\textbf{One-sided entropic \texorpdfstring{$c$}{c}-transform.}

\begin{proposition}
\label{prop:one_marginal}
Fix a positive key marginal $\nu$ and $f \in \mathbb{R}^N$, and define
\begin{equation}
g = T_K^\nu(f),
\qquad
P_f{}_{ij}
=
\exp\left(\frac{-C_{ij} + f_i + g_j}{\varepsilon}\right).
\end{equation}
The plan then satisfies $P_f^\top \mathbf{1}_N = \nu$. Moreover, for any constant $c \in \mathbb{R}$,
\begin{equation}
T_K^\nu(f + c\mathbf{1}_N) = T_K^\nu(f) - c\mathbf{1}_N,
\qquad
P_{f + c\mathbf{1}_N} = P_f.
\end{equation}
\end{proposition}

\begin{proof}
We exponentiate the definition of the key-side $c$-transform to get
\begin{equation}
\exp\left(\frac{g_j}{\varepsilon}\right)
=
\frac{\nu_j}{\sum_{i=1}^N \exp\left(\frac{-C_{ij}+f_i}{\varepsilon}\right)}.
\end{equation}
The $c$-transform was chosen so that this denominator cancels inside each key column. For key $j$, we have
\begin{equation}
\sum_{i=1}^N P_f{}_{ij}
=
\exp\left(\frac{g_j}{\varepsilon}\right)
\sum_{i=1}^N
\exp\left(\frac{-C_{ij}+f_i}{\varepsilon}\right)
=
\nu_j,
\end{equation}
which proves $P_f^\top \mathbf{1}_N = \nu$.

Now let $c\in\mathbb{R}$. The constant shift in the source dual appears as a common factor inside the logarithm in $T_K^\nu(f+c\mathbf{1}_N)_j$, so
\begin{equation}
T_K^\nu(f+c\mathbf{1}_N)_j
=
\varepsilon \log \nu_j
- \varepsilon \log
\left(
\exp(c/\varepsilon)
\sum_{i=1}^N
\exp\left(\frac{-C_{ij}+f_i}{\varepsilon}\right)
\right)
=
T_K^\nu(f)_j - c.
\end{equation}
When we place this shifted key dual back into the plan, the two constants cancel:
\begin{equation}
P_{f+c\mathbf{1}_N}{}_{ij}
=
\exp\left(
\frac{-C_{ij} + (f_i+c) + (g_j-c)}{\varepsilon}
\right)
=
P_f{}_{ij},
\end{equation}
which proves additive-shift invariance.
\end{proof}

\begin{proposition}
\label{prop:source_dual_identifiability}
Fix $C$, $\varepsilon>0$, and a positive key marginal $\nu$. For source duals $f,\widehat f\in\mathbb{R}^N$, let $P_f$ and $P_{\widehat f}$ be the key-normalized plans obtained from $T_K^\nu(f)$ and $T_K^\nu(\widehat f)$. Under this notation, we have
\begin{equation}
P_{\widehat f}=P_f
\quad\Longleftrightarrow\quad
\widehat f-f\in\operatorname{span}\{\mathbf{1}_N\}.
\end{equation}
In particular, if $f$ and $\widehat f$ are centered, then $P_{\widehat f}=P_f$ if and only if $\widehat f=f$.
\end{proposition}

\begin{proof}
One direction is exactly the additive-shift invariance in Proposition~\ref{prop:one_marginal}. We prove the converse. Suppose $P_{\widehat f}=P_f$. Once the key marginal is fixed, each key column is a normalized distribution over queries. For each key $j$, the key-side $c$-transform lets us write
\begin{equation}
P_f{}_{ij}
=
\nu_j
\frac{\exp((-C_{ij}+f_i)/\varepsilon)}
{\sum_r \exp((-C_{rj}+f_r)/\varepsilon)}
\end{equation}
and the same formula holds with $f$ replaced by $\widehat f$. Since $\nu_j>0$ and the two plans are equal, the normalized distributions within each key column are equal. Therefore, for every $i$ and $j$,
\begin{equation}
\frac{\exp((-C_{ij}+\widehat f_i)/\varepsilon)}
{\exp((-C_{ij}+f_i)/\varepsilon)}
\end{equation}
is independent of $i$ within that column. The cost terms cancel, so $\exp((\widehat f_i-f_i)/\varepsilon)$ is independent of $i$. Hence $\widehat f-f=c\mathbf{1}_N$ for some constant $c$. If both vectors are centered, the constant has mean zero and $c=0$.
\end{proof}

\begin{proposition}
\label{prop:key_projection}
Fix $C$, $\varepsilon>0$, a source dual $f\in\mathbb{R}^N$, and a positive key marginal $\nu$. Define
\begin{equation}
R^f_{ij}
=
\exp\left(\frac{-C_{ij}+f_i}{\varepsilon}\right)
\end{equation}
and
\begin{equation}
\mathcal{K}_\nu
=
\{P\in\mathbb{R}^{N\times N}_+:P^\top\mathbf{1}_N=\nu\}.
\end{equation}
We use the generalized KL divergence below because $R^f$ is unnormalized.
For nonnegative $P$ and positive $R$, use the generalized KL divergence
\begin{equation}
D_{\mathrm{KL}}^{\mathrm{gen}}(P\|R)
=
\sum_{i,j}
\left[
P_{ij}\log\frac{P_{ij}}{R_{ij}}-P_{ij}+R_{ij}
\right].
\end{equation}
The key-normalized Gibbs plan $P_f=\mathcal{C}_\nu(f;C)$ is the unique minimizer of
\begin{equation}
\min_{P\in\mathcal{K}_\nu}D_{\mathrm{KL}}^{\mathrm{gen}}(P\|R^f).
\end{equation}
\end{proposition}

\begin{proof}
The key constraint is columnwise, so the projection can be solved one key at a time. For a fixed key $j$, we minimize
\begin{equation}
\sum_i
\left[
p_i\log\frac{p_i}{R^f_{ij}}-p_i+R^f_{ij}
\right]
\end{equation}
subject to $p_i\ge0$ and $\sum_i p_i=\nu_j$. The objective is strictly convex on the positive orthant. With a Lagrange multiplier $\lambda_j$ for the column-mass constraint, stationarity requires
\begin{equation}
\log\frac{p_i}{R^f_{ij}}+\lambda_j=0.
\end{equation}
Since $R^f_{ij}>0$ and $\nu_j>0$, the minimizer lies in the positive orthant, so this stationarity condition characterizes it. We have $p_i=R^f_{ij}\exp(-\lambda_j)$, and the value of $\lambda_j$ is determined by the column sum:
\begin{equation}
p_i
=
\nu_j\frac{R^f_{ij}}{\sum_r R^f_{rj}}
=
\exp\left(\frac{-C_{ij}+f_i+T_K^\nu(f)_j}{\varepsilon}\right).
\end{equation}
We have recovered the key-side $c$-transform. Strict convexity on each positive column ensures uniqueness, and the columnwise minimizers assemble into the unique key-normalized plan.
\end{proof}

We write $\mathcal{C}_\nu(f;C)$ for this key-normalized Gibbs plan.
This is the plan used by ASAP-0 and by the convex KL calibration objective below.

\medskip\noindent\textbf{Column-ending two-sided entropic \texorpdfstring{$c$}{c}-transform.}

\begin{proposition}
\label{prop:two_sided_transform}
Fix positive marginals $\mu,\nu\in\mathbb{R}^N_+$ with $\mathbf{1}_N^\top\mu=\mathbf{1}_N^\top\nu=1$. For a source dual $f\in\mathbb{R}^N$, define
\begin{equation}
\begin{aligned}
g^0 &= T_K^\nu(f),\\
f^+ &= T_Q^\mu(g^0),\\
g^+ &= T_K^\nu(f^+),
\end{aligned}
\end{equation}
where
\begin{equation}
T_Q^\mu(g)_i
=
\varepsilon\log\mu_i
-
\varepsilon\log
\sum_{j=1}^N
\exp\left(\frac{-C_{ij}+g_j}{\varepsilon}\right).
\end{equation}
Define the intermediate and final plans
\begin{equation}
R^f_{ij}
=
\exp\left(\frac{-C_{ij}+f^+_i+g^0_j}{\varepsilon}\right),
\qquad
P^{\leftrightarrow}_{f,ij}
=
\exp\left(\frac{-C_{ij}+f^+_i+g^+_j}{\varepsilon}\right).
\end{equation}
Then
\begin{equation}
R^f\mathbf{1}_N=\mu,
\qquad
\bigl(P_f^{\leftrightarrow}\bigr)^\top\mathbf{1}_N=\nu.
\end{equation}
Moreover, for any constant $c\in\mathbb{R}$,
\begin{equation}
P^{\leftrightarrow}_{f+c\mathbf{1}_N}
=
P^{\leftrightarrow}_{f}.
\end{equation}
\end{proposition}

\begin{proof}
We choose the row-side transform so that each query row of the intermediate plan has the prescribed mass. By definition,
\begin{equation}
\exp\left(\frac{f^+_i}{\varepsilon}\right)
=
\frac{\mu_i}{\sum_j \exp\left(\frac{-C_{ij}+g^0_j}{\varepsilon}\right)}.
\end{equation}
We substitute this expression into the row sum of $R^f$ and get
\begin{equation}
\sum_j R^f_{ij}
=
\exp\left(\frac{f^+_i}{\varepsilon}\right)
\sum_j
\exp\left(\frac{-C_{ij}+g^0_j}{\varepsilon}\right)
=
\mu_i.
\end{equation}
We then apply the same column argument as Proposition~\ref{prop:one_marginal}, now with the updated source dual $f^+$, and obtain
\begin{equation}
\sum_i P^{\leftrightarrow}_{f,ij}
=
\nu_j.
\end{equation}

For shift invariance, let $f'=f+c\mathbf{1}_N$. Proposition~\ref{prop:one_marginal} yields $g^{0\prime}=g^0-c\mathbf{1}_N$. The source-side transform moves by the opposite amount:
\begin{equation}
T_Q^\mu(g^0-c\mathbf{1}_N)=T_Q^\mu(g^0)+c\mathbf{1}_N,
\end{equation}
so $f^{+\prime}=f^+ + c\mathbf{1}_N$. We apply Proposition~\ref{prop:one_marginal} one more time and get $g^{+\prime}=g^+-c\mathbf{1}_N$. The constants cancel inside the exponential expression,
\begin{equation}
\exp\left(\frac{-C_{ij}+f^{+\prime}_i+g^{+\prime}_j}{\varepsilon}\right)
=
\exp\left(\frac{-C_{ij}+f^+_i+g^+_j}{\varepsilon}\right),
\end{equation}
which proves $P^{\leftrightarrow}_{f+c\mathbf{1}_N}=P^{\leftrightarrow}_{f}$.
\end{proof}

The proposition states a row-marginal property for the intermediate plan $R^f$ and a column-marginal property for the returned plan $P_f^{\leftrightarrow}$. It does not assert that $P_f^{\leftrightarrow}$ satisfies both marginals exactly. The row-ending transform used for teachers whose last Sinkhorn step normalizes rows returns $R^f$ instead.

\begin{proposition}
\label{prop:two_sided_row_error}
Under the notation of Proposition~\ref{prop:two_sided_transform}, set
\begin{equation}
\delta_f=\|g^+-g^0\|_\infty .
\end{equation}
Then the row error of the final column-marginal plan obeys
\begin{equation}
\left\|P_f^{\leftrightarrow}\mathbf{1}_N-\mu\right\|_1
\le
\left\|P_f^{\leftrightarrow}-R^f\right\|_1
\le
\exp(\delta_f/\varepsilon)-1.
\end{equation}
The matrix norm in the middle term is entrywise $\ell_1$.
\end{proposition}

\begin{proof}
By Proposition~\ref{prop:two_sided_transform}, the intermediate plan satisfies $R^f\mathbf{1}_N=\mu$. Row summation cannot increase entrywise $\ell_1$ distance, so
\begin{equation}
\left\|P_f^{\leftrightarrow}\mathbf{1}_N-\mu\right\|_1
=
\left\|\bigl(P_f^{\leftrightarrow}-R^f\bigr)\mathbf{1}_N\right\|_1
\le
\left\|P_f^{\leftrightarrow}-R^f\right\|_1.
\end{equation}
Next compare the final plan with the intermediate plan. They are related by a columnwise factor:
\begin{equation}
P^{\leftrightarrow}_{f,ij}
=
R^f_{ij}
\exp\left(\frac{g^+_j-g^0_j}{\varepsilon}\right).
\end{equation}
Since $|g^+_j-g^0_j|\le\delta_f$, we have
\begin{equation}
\left|
\exp\left(\frac{g^+_j-g^0_j}{\varepsilon}\right)-1
\right|
\le
\exp(\delta_f/\varepsilon)-1.
\end{equation}
We multiply this uniform bound by $\sum_{i,j}R^f_{ij}=1$, which proves the claim.
\end{proof}

\medskip\noindent\textbf{Dual-error stability.}

\begin{proposition}
\label{prop:dual_stability}
Fix $C$, $\varepsilon > 0$, and a positive key marginal $\nu$. For $f,\widehat f \in \mathbb{R}^N$, define
\begin{equation}
g = T_K^\nu(f), \qquad \widehat g = T_K^\nu(\widehat f),
\end{equation}
and let $P_f$ and $P_{\widehat f}$ be the corresponding reconstructed plans. If
\begin{equation}
\|\widehat f - f\|_\infty \le \alpha,
\end{equation}
then
\begin{equation}
\|\widehat g - g\|_\infty \le \alpha
\end{equation}
and, for every $i,j$,
\begin{equation}
e^{-2\alpha/\varepsilon} P_{f,ij}
\le
P_{\widehat f,ij}
\le
e^{2\alpha/\varepsilon} P_{f,ij}.
\end{equation}
\end{proposition}

\begin{proof}
Let $\Delta f=\widehat f-f$. Fix a key $j$. The two key-side $c$-transforms change only through the vector inside one log-sum-exp. Log-sum-exp is $1$-Lipschitz in $\ell_\infty$, and the two input vectors are at most $\|\Delta f\|_\infty/\varepsilon$ apart. Therefore
\begin{equation}
|\widehat g_j - g_j|
\le
\varepsilon \cdot \frac{\|\Delta f\|_\infty}{\varepsilon}
=
\|\Delta f\|_\infty
\le
\alpha.
\end{equation}
This holds for every key, so the whole key dual moves by at most $\alpha$, and $\|\widehat g-g\|_\infty\le\alpha$.

Now compare the two plans entry by entry. Their ratio contains only the source-dual error and the induced key-dual error:
\begin{equation}
\frac{P_{\widehat f,ij}}{P_{f,ij}}
=
\exp\left(
\frac{(\widehat f_i-f_i)+(\widehat g_j-g_j)}{\varepsilon}
\right).
\end{equation}
Because $|\widehat f_i-f_i| \le \alpha$ and $|\widehat g_j-g_j| \le \alpha$, the exponent lies in the interval $[ -2\alpha/\varepsilon,\, 2\alpha/\varepsilon]$. Hence
\begin{equation}
e^{-2\alpha/\varepsilon}
\le
\frac{P_{\widehat f,ij}}{P_{f,ij}}
\le
e^{2\alpha/\varepsilon},
\end{equation}
which is equivalent to
\begin{equation}
e^{-2\alpha/\varepsilon} P_{f,ij}
\le
P_{\widehat f,ij}
\le
e^{2\alpha/\varepsilon} P_{f,ij}.
\end{equation}
\end{proof}

Because the reconstruction is invariant to additive shifts, the plan-ratio bound also holds with $\alpha$ replaced by $\inf_{c \in \mathbb{R}}\|\widehat f-f-c\mathbf{1}_N\|_\infty$. The key-dual bound itself is a statement about a fixed representative of the source dual. In the main construction, both teacher and predicted source duals are centered before fitting.

\begin{proposition}
\label{prop:two_sided_stability}
Let $P_f^{\leftrightarrow}$ and $P_{\widehat f}^{\leftrightarrow}$ be the two-sided plans from Proposition~\ref{prop:two_sided_transform}. If
\begin{equation}
\|\widehat f-f\|_\infty\le\alpha,
\end{equation}
then every dual produced by the two-sided transform moves by at most $\alpha$:
\begin{equation}
\|\widehat g^0-g^0\|_\infty\le\alpha,
\qquad
\|\widehat f^+-f^+\|_\infty\le\alpha,
\qquad
\|\widehat g^+-g^+\|_\infty\le\alpha.
\end{equation}
Consequently, for every $i,j$,
\begin{equation}
e^{-2\alpha/\varepsilon}P^{\leftrightarrow}_{f,ij}
\le
P^{\leftrightarrow}_{\widehat f,ij}
\le
e^{2\alpha/\varepsilon}P^{\leftrightarrow}_{f,ij}.
\end{equation}
\end{proposition}

\begin{proof}
Proposition~\ref{prop:dual_stability} yields the key-side bound $\|\widehat g^0-g^0\|_\infty\le\alpha$. The same log-sum-exp Lipschitz argument applies to the source-side transform $T_Q^\mu$, so $\|\widehat f^+-f^+\|_\infty\le\|\widehat g^0-g^0\|_\infty\le\alpha$. We apply Proposition~\ref{prop:dual_stability} again to $f^+$ and $\widehat f^+$ and get $\|\widehat g^+-g^+\|_\infty\le\alpha$.

We now compare the two final plans:
\begin{equation}
\frac{P^{\leftrightarrow}_{\widehat f,ij}}{P^{\leftrightarrow}_{f,ij}}
=
\exp\left(
\frac{(\widehat f^+_i-f^+_i)+(\widehat g^+_j-g^+_j)}{\varepsilon}
\right).
\end{equation}
Both terms in the numerator are bounded in absolute value by $\alpha$, which proves the displayed multiplicative bound.
\end{proof}

The proposition compares outputs of the same two-sided map. When the finite teacher plan is defined by the one-sided final closure, default ASAP has an additional finite-correction bias term such as $\|P^{\leftrightarrow}_{f^{(T)}}-P^{(T)}\|$. ASAP-0 has zero such term under exact prediction in the column-ending final-closure case.

\begin{proposition}
\label{prop:kl_stability}
Fix $C$, $\varepsilon>0$, and a positive key marginal $\nu$ with $\sum_j\nu_j=1$. Let $P=P_f$ and $\widehat P=P_{\widehat f}$ be the key-normalized Gibbs plans induced by source duals $f$ and $\widehat f$ under the same cost and key marginal. For each key $j$, set $\pi_j(i)=P_{ij}/\nu_j$, $\widehat\pi_j(i)=\widehat P_{ij}/\nu_j$, and $\Delta_i=\widehat f_i-f_i$.
For $t\in[0,1]$, define the tilted column distribution
\begin{equation}
\pi_{j,t}(i)
=
\frac{\pi_j(i)\exp(t\Delta_i/\varepsilon)}
{\sum_r \pi_j(r)\exp(t\Delta_r/\varepsilon)}.
\end{equation}
The reverse KL identity is
\begin{equation}
D_{\mathrm{KL}}(\widehat P\|P)
=
\sum_j \nu_j
\int_0^1
t\,\operatorname{Var}_{\pi_{j,t}}(\Delta/\varepsilon)\,dt.
\end{equation}
The forward direction satisfies the companion identity
\begin{equation}
D_{\mathrm{KL}}(P\|\widehat P)
=
\sum_j \nu_j
\int_0^1
(1-t)\,\operatorname{Var}_{\pi_{j,t}}(\Delta/\varepsilon)\,dt.
\end{equation}
The two identities use two weights on the same interpolation; the common quadratic bound below comes from a uniform ceiling on the variance over $t$.
Consequently, if $\operatorname{Var}_{\pi_{j,t}}(\Delta)\le \sigma_j^2$ for all $t\in[0,1]$, then
\begin{equation}
D_{\mathrm{KL}}(\widehat P\|P),
\;
D_{\mathrm{KL}}(P\|\widehat P)
\le
\frac{1}{2\varepsilon^2}\sum_j\nu_j\sigma_j^2.
\end{equation}
If, additionally, $\pi_{j,t}(i)\le \beta/N$ for all $i,j,t$, then
\begin{equation}
D_{\mathrm{KL}}(\widehat P\|P),
\;
D_{\mathrm{KL}}(P\|\widehat P)
\le
\frac{\beta}{2\varepsilon^2 N}\|\widehat f-f\|_2^2.
\end{equation}
The constant $\beta$ measures column sharpness: $\beta=1$ corresponds to uniform column conditionals, while sharper column conditionals give larger $\beta$.
\end{proposition}

\begin{proof}
For each key $j$, the key-side $c$-transform turns the new column conditional into an exponential tilt of the old one. In particular,
\begin{equation}
\widehat \pi_j(i)
=
\frac{\pi_j(i)\exp(\Delta_i/\varepsilon)}
{\sum_r \pi_j(r)\exp(\Delta_r/\varepsilon)}.
\end{equation}
Let
\begin{equation}
A_j(t)=\log\sum_i \pi_j(i)\exp(t\Delta_i/\varepsilon).
\end{equation}
This is the log-partition function along the interpolation. Its first derivative is the tilted mean of $\Delta/\varepsilon$, and its second derivative is the tilted variance:
\begin{equation}
A_j'(t)=\sum_i \pi_{j,t}(i)\Delta_i/\varepsilon,
\qquad
A_j''(t)=\operatorname{Var}_{\pi_{j,t}}(\Delta/\varepsilon).
\end{equation}
At $t=1$, the interpolation reaches $\widehat\pi_j$. We can write the reverse column KL as
\begin{equation}
D_{\mathrm{KL}}(\widehat \pi_j\|\pi_j)
=
\sum_i \pi_{j,1}(i)\left(\Delta_i/\varepsilon-A_j(1)\right)
=
A_j'(1)-A_j(1).
\end{equation}
Since $A_j(0)=0$, this becomes
\begin{equation}
A_j'(1)-A_j(1)
=
A_j'(1)-\int_0^1 A_j'(s)\,ds
=
\int_0^1 t A_j''(t)\,dt.
\end{equation}
The forward column KL uses the same interpolation, with the weight reversed:
\begin{equation}
D_{\mathrm{KL}}(\pi_j\|\widehat \pi_j)
=
\sum_i \pi_j(i)\left(A_j(1)-\Delta_i/\varepsilon\right)
=
A_j(1)-A_j'(0).
\end{equation}
Since $A_j(1)=\int_0^1 A_j'(s)\,ds$, we obtain
\begin{equation}
A_j(1)-A_j'(0)
=
\int_0^1 (1-t)A_j''(t)\,dt.
\end{equation}
Returning to plan scale, $P_{ij}=\nu_j\pi_j(i)$ and $\widehat P_{ij}=\nu_j\widehat\pi_j(i)$. Therefore $D_{\mathrm{KL}}(\widehat P\|P)=\sum_j\nu_jD_{\mathrm{KL}}(\widehat\pi_j\|\pi_j)$ and $D_{\mathrm{KL}}(P\|\widehat P)=\sum_j\nu_jD_{\mathrm{KL}}(\pi_j\|\widehat\pi_j)$. Summing over columns proves both identities. The first bounds follow from the variance ceiling and the identities $\int_0^1 t\,dt=\int_0^1(1-t)\,dt=1/2$. For the last bounds, the density condition implies
\begin{equation}
\operatorname{Var}_{\pi_{j,t}}(\Delta)
\le
\sum_i \pi_{j,t}(i)\Delta_i^2
\le
\frac{\beta}{N}\|\Delta\|_2^2,
\end{equation}
and $\sum_j\nu_j=1$. We plug this estimate into the previous bound to finish the proof.
\end{proof}

\noindent\textbf{One-sided KL calibration.}
The next propositions concern the key-normalized plan $\mathcal{C}_\nu(f;C)$. They justify the convex ASAP-KL calibration surrogate for the one-sided plan; they do not claim convexity of the full two-sided ASAP inference map.

\begin{proposition}
\label{prop:asap_kl_stationarity}
For fit examples $m=1,\ldots,M$, let
\begin{equation}
\label{eq:asap_kl_scalar}
F_m(f)
=
\varepsilon\sum_j\nu_j
\log\sum_i\exp\left(\frac{-C^{(m)}_{ij}+f_i}{\varepsilon}\right),
\qquad
\widehat P_{\omega}^{0,(m)}
=
\mathcal{C}_{\nu}(X_m\omega;C^{(m)}),
\end{equation}
and let $r^{(T,m)}=P^{(T,m)}\mathbf{1}_N$. Assume $P^{(T,m)\top}\mathbf{1}_N=\nu$ for each fit example. For $\lambda>0$, define
\begin{equation}
\label{eq:asap_kl}
J_{\mathrm{KL}}(\omega)
=
\sum_{m=1}^M
\left[
F_m(X_m\omega)-\langle r^{(T,m)},X_m\omega\rangle
\right]
+\frac{\lambda}{2}\|\omega\|_2^2.
\end{equation}
The objective $J_{\mathrm{KL}}$ is strongly convex and has a unique minimizer $\omega^\star$. This minimizer satisfies
\begin{equation}
\sum_{m=1}^M
X_m^\top
\left(
\widehat P_{\omega^\star}^{0,(m)}\mathbf{1}_N-r^{(T,m)}
\right)
=
-\lambda\omega^\star.
\end{equation}
Moreover, for $f_m=X_m\omega$ and $\pi^{(m)}_j(i)=\mathcal{C}_{\nu}(f_m;C^{(m)})_{ij}/\nu_j$,
\begin{equation}
\nabla_\omega^2 J_{\mathrm{KL}}(\omega)
=
\sum_{m=1}^M
X_m^\top
\left[
\frac{1}{\varepsilon}
\sum_j \nu_j
\left(
\operatorname{Diag}(\pi^{(m)}_j)
-\pi^{(m)}_j(\pi^{(m)}_j)^\top
\right)
\right]
X_m
+\lambda I_L
\succeq
\lambda I_L.
\end{equation}
\end{proposition}

\begin{proof}
The convexity comes from the log-sum-exp terms. Each $F_m$ is composed with the linear map $\omega\mapsto X_m\omega$, so $F_m(X_m\omega)$ is convex in $\omega$. The term $-\langle r^{(T,m)},X_m\omega\rangle$ is affine, and $\lambda\|\omega\|_2^2/2$ is $\lambda$-strongly convex. Therefore $J_{\mathrm{KL}}$ is strongly convex and has a unique minimizer.

We next compute the gradient. For a vector $f$, the derivative of $F_m$ with respect to the source coordinate $f_i$ is
\begin{equation}
\frac{\partial F_m}{\partial f_i}(f)
=
\sum_j
\nu_j
\frac{
\exp((-C^{(m)}_{ij}+f_i)/\varepsilon)
}{
\sum_r\exp((-C^{(m)}_{rj}+f_r)/\varepsilon)
}.
\end{equation}
The right-hand side is the $i$th row mass of $\mathcal{C}_{\nu}(f;C^{(m)})$. At $f=X_m\omega$, the chain rule yields
\begin{equation}
\nabla_\omega J_{\mathrm{KL}}(\omega)
=
\sum_{m=1}^M
X_m^\top
\left(
\widehat P_{\omega}^{0,(m)}\mathbf{1}_N-r^{(T,m)}
\right)
+\lambda\omega.
\end{equation}
At the unique minimizer this gradient is zero, which proves the displayed identity. We take one more derivative of the row-mass expression. For $f=X_m\omega$,
\begin{equation}
\nabla_f^2 F_m(f)
=
\frac{1}{\varepsilon}
\sum_j\nu_j
\left(
\operatorname{Diag}(\pi^{(m)}_j)
-\pi^{(m)}_j(\pi^{(m)}_j)^\top
\right).
\end{equation}
Each matrix inside the sum is a covariance matrix under the column conditional $\pi^{(m)}_j$, so it is positive semidefinite. The chain rule yields the displayed Hessian for $J_{\mathrm{KL}}$, and the ridge term leaves the whole Hessian bounded below by $\lambda I_L$.
\end{proof}

\begin{proposition}
\label{prop:ls_kl_surrogate}
For each fit example $m$, let $y_m$ be the centered teacher source dual and let $P^{(T,m)}=\mathcal{C}_{\nu}(y_m;C^{(m)})$. For any coefficient vector $\omega$, set $\widehat f_m=X_m\omega$ and $\widehat P_{\omega}^{0,(m)}=\mathcal{C}_{\nu}(\widehat f_m;C^{(m)})$. Assume the density condition in Proposition~\ref{prop:kl_stability} holds with the same constant $\beta$ for each pair $(\widehat P_{\omega}^{0,(m)},P^{(T,m)})$. Under this assumption, we have
\begin{equation}
\frac{1}{M}
\sum_{m=1}^M
D_{\mathrm{KL}}\!\left(\widehat P_{\omega}^{0,(m)}\|P^{(T,m)}\right)
\le
\frac{\beta}{2\varepsilon^2 M}
\sum_{m=1}^M
\|X_m\omega-y_m\|_N^2.
\end{equation}
Thus, conditional on the density assumption above, the ridge least-squares fit used by ASAP-LS minimizes an empirical quadratic upper bound on the average reverse teacher-plan KL over the chosen sliced-dual feature span. The minimizer of this surrogate need not coincide with the minimizer of the reverse plan KL itself; ASAP-KL in Proposition~\ref{prop:asap_kl_stationarity} minimizes the forward KL for the one-sided key-normalized plan through the same source-dual parameterization.
\end{proposition}

\begin{proof}
Fix a calibration example $m$ and apply the last bound in Proposition~\ref{prop:kl_stability} with $f=y_m$ and $\widehat f=X_m\omega$. Since $\|u\|_N^2=N^{-1}\|u\|_2^2$, we have
\begin{equation}
D_{\mathrm{KL}}\!\left(\widehat P_{\omega}^{0,(m)}\|P^{(T,m)}\right)
\le
\frac{\beta}{2\varepsilon^2}
\|X_m\omega-y_m\|_N^2.
\end{equation}
We apply the same bound on every calibration example and average.
\end{proof}

\medskip\noindent\textbf{A conditional random-feature bound.}

Proposition~\ref{prop:kl_stability} shows how source-dual error controls the key-normalized Gibbs plan. We also need a modest statement about the sliced features themselves. We state the next result as an oracle bound. It assumes the centered teacher dual is already close to a square-integrable average of sliced potentials, then shows that independent slices recover that average at the usual Monte Carlo rate; it does not claim that arbitrary Sinkhorn teacher duals can be approximated by finitely many slices.

\begin{proposition}
\label{prop:rf_oracle}
Let $Z=(Q,K)$ be drawn from a calibration distribution, and let $y(Z)\in\mathbb{R}^N$ be the centered teacher source dual. Write $\|u\|_N^2=N^{-1}\|u\|_2^2$. For each slice direction $\theta\sim\sigma$, let $\varphi_\theta(Z)\in\mathbb{R}^N$ be the centered sliced-potential feature. Assume that
\begin{equation}
y(Z)
=
\mathbb{E}_{\theta\sim\sigma}\!\left[a(\theta)\varphi_\theta(Z)\right]
+r(Z),
\end{equation}
where
\begin{equation}
\mathbb{E}_{\theta\sim\sigma}a(\theta)^2\le A^2,
\qquad
\|\varphi_\theta(Z)\|_N\le B,
\qquad
\mathbb{E}_{Z}\|r(Z)\|_N^2\le \eta^2 .
\end{equation}
Draw $\theta_1,\ldots,\theta_L$ independently from $\sigma$, and define the oracle coefficients $\omega^\star_\ell=a(\theta_\ell)/L$. If
\begin{equation}
R_\Theta(\omega)
=
\mathbb{E}_{Z}
\left\|
\sum_{\ell=1}^L \omega_\ell\varphi_{\theta_\ell}(Z)-y(Z)
\right\|_N^2,
\end{equation}
then
\begin{equation}
\mathbb{E}_{\Theta}R_\Theta(\omega^\star)
\le
\eta^2+\frac{A^2B^2}{L}.
\end{equation}
If, in addition, the density condition in Proposition~\ref{prop:kl_stability} holds with constant $\beta$ for the key-normalized Gibbs plans induced by $y(Z)$ and by $\sum_\ell\omega^\star_\ell\varphi_{\theta_\ell}(Z)$, then
\begin{equation}
\mathbb{E}_{\Theta,Z}
D_{\mathrm{KL}}\!\left(\widehat P_{\omega^\star}(Z)\,\|\,P^{(T)}(Z)\right)
\le
\frac{\beta}{2\varepsilon^2}
\left(
\eta^2+\frac{A^2B^2}{L}
\right).
\end{equation}
\end{proposition}

\begin{proof}
Let
\begin{equation}
h(Z)=\mathbb{E}_{\theta\sim\sigma}\!\left[a(\theta)\varphi_\theta(Z)\right],
\qquad
h_L(Z)=\frac{1}{L}\sum_{\ell=1}^L a(\theta_\ell)\varphi_{\theta_\ell}(Z).
\end{equation}
Here $h$ is the population sliced-potential average and $h_L$ is its Monte Carlo approximation from the sampled slices. For a fixed input $Z$, write
\begin{equation}
\xi_\ell(Z)=a(\theta_\ell)\varphi_{\theta_\ell}(Z)-h(Z).
\end{equation}
The vectors $\xi_1(Z),\ldots,\xi_L(Z)$ are independent and have mean zero. The slice noise therefore has mean zero, so the mixed terms disappear in expectation and
\begin{equation}
\mathbb{E}_{\Theta}\|h_L(Z)-h(Z)\|_N^2
=
\frac{1}{L}
\mathbb{E}_{\theta\sim\sigma}
\left\|
a(\theta)\varphi_\theta(Z)-h(Z)
\right\|_N^2.
\end{equation}
The last term is bounded by the second moment:
\begin{equation}
\mathbb{E}_{\Theta}\|h_L(Z)-h(Z)\|_N^2
\le
\frac{1}{L}
\mathbb{E}_{\theta\sim\sigma}
a(\theta)^2\|\varphi_\theta(Z)\|_N^2
\le
\frac{A^2B^2}{L}.
\end{equation}
Since $y(Z)=h(Z)+r(Z)$ and $\mathbb{E}_{\Theta}[h_L(Z)-h(Z)]=0$, the cross term with $r(Z)$ also vanishes. We have
\begin{equation}
\mathbb{E}_{\Theta}
\|h_L(Z)-y(Z)\|_N^2
=
\mathbb{E}_{\Theta}
\|h_L(Z)-h(Z)\|_N^2
+\|r(Z)\|_N^2.
\end{equation}
Averaging over $Z$ proves the displayed bound on $\mathbb{E}_\Theta R_\Theta(\omega^\star)$.

We now prove the KL bound. For each $Z$, apply Proposition~\ref{prop:kl_stability} with
\begin{equation}
\widehat f(Z)=h_L(Z),
\qquad
f(Z)=y(Z).
\end{equation}
Since $\|\widehat f(Z)-f(Z)\|_N^2=N^{-1}\|\widehat f(Z)-f(Z)\|_2^2$, we obtain
\begin{equation}
D_{\mathrm{KL}}\!\left(\widehat P_{\omega^\star}(Z)\,\|\,P^{(T)}(Z)\right)
\le
\frac{\beta}{2\varepsilon^2}
\|\widehat f(Z)-f(Z)\|_N^2.
\end{equation}
After averaging over $Z$ and the slice draw, the source-dual error bound above proves the KL statement.
\end{proof}

\begin{proposition}
\label{prop:ls_empirical_oracle}
Let $Z_1,\ldots,Z_M$ be calibration examples, let $X_m=X_\Theta(Z_m)$ be the sliced-feature matrix for example $m$, and let $y_m=y(Z_m)$ be the centered teacher source dual. Define
\begin{equation}
R_M(\omega)
=
\frac{1}{M}
\sum_{m=1}^M
\|X_m\omega-y_m\|_N^2 .
\end{equation}
For $\lambda_{\mathrm{risk}}\ge0$, let $\widehat\omega_{\mathrm{LS}}$ be any minimizer of
\begin{equation}
R_M(\omega)+\lambda_{\mathrm{risk}}\|\omega\|_2^2 .
\end{equation}
For every comparator $\omega^\circ\in\mathbb{R}^L$, we have
\begin{equation}
R_M(\widehat\omega_{\mathrm{LS}})
\le
R_M(\omega^\circ)+\lambda_{\mathrm{risk}}\|\omega^\circ\|_2^2 .
\end{equation}
Now draw the slices as in Proposition~\ref{prop:rf_oracle} and set $\omega^\circ_\ell=a(\theta_\ell)/L$. Under the assumptions of Proposition~\ref{prop:rf_oracle},
\begin{equation}
\mathbb{E}_{\Theta,Z_{1:M}}
R_M(\widehat\omega_{\mathrm{LS}})
\le
\eta^2+\frac{A^2B^2}{L}+\lambda_{\mathrm{risk}}\frac{A^2}{L}.
\end{equation}
If the density condition in Proposition~\ref{prop:kl_stability} holds with constant $\beta$ for the corresponding key-normalized Gibbs plans, then
\begin{equation}
\mathbb{E}_{\Theta,Z_{1:M}}
\frac{1}{M}
\sum_{m=1}^M
D_{\mathrm{KL}}\!\left(
\widehat P_{\widehat\omega_{\mathrm{LS}}}^{(m)}
\|P^{(T,m)}
\right)
\le
\frac{\beta}{2\varepsilon^2}
\left(
\eta^2+\frac{A^2B^2}{L}+\lambda_{\mathrm{risk}}\frac{A^2}{L}
\right).
\end{equation}
The proposition controls the empirical calibration objective optimized by ASAP-LS. It is not a test-distribution generalization bound.
\end{proposition}

\begin{proof}
By optimality of $\widehat\omega_{\mathrm{LS}}$, its regularized empirical risk is no larger than that of any comparator. Therefore, for every $\omega^\circ$ we have
\begin{equation}
R_M(\widehat\omega_{\mathrm{LS}})
+
\lambda_{\mathrm{risk}}\|\widehat\omega_{\mathrm{LS}}\|_2^2
\le
R_M(\omega^\circ)+\lambda_{\mathrm{risk}}\|\omega^\circ\|_2^2 .
\end{equation}
The left ridge term can only make the left side larger, so removing it keeps the inequality valid and proves the first claim.

Now choose the random-feature oracle comparator $\omega^\circ_\ell=a(\theta_\ell)/L$. The same calculation as in Proposition~\ref{prop:rf_oracle} shows
\begin{equation}
\mathbb{E}_{\Theta,Z_{1:M}}R_M(\omega^\circ)
\le
\eta^2+\frac{A^2B^2}{L}.
\end{equation}
For the ridge term, we have
\begin{equation}
\mathbb{E}_\Theta\|\omega^\circ\|_2^2
=
\frac{1}{L^2}
\sum_{\ell=1}^L
\mathbb{E}_{\theta_\ell}a(\theta_\ell)^2
\le
\frac{A^2}{L}.
\end{equation}
After averaging the first inequality, we obtain the empirical risk bound.

For the KL statement, we apply Proposition~\ref{prop:ls_kl_surrogate} to $\omega=\widehat\omega_{\mathrm{LS}}$ and take expectation over the calibration examples and the slice draw.
\end{proof}

\begin{proposition}
\label{prop:row_error}
In the unmasked square setting, let $\mu=N^{-1}\mathbf{1}_N$ and define $\rho(P)=\|P\mathbf{1}_N-\mu\|_1$. For two key-normalized probability plans $P$ and $\widehat P$ with the same key marginal, using entrywise $\ell_1$ for matrix gaps,
\begin{equation}
\rho(\widehat P)
\le
\rho(P)+\|\widehat P-P\|_1
\le
\rho(P)+\sqrt{2D_{\mathrm{KL}}(\widehat P\|P)}.
\end{equation}
For attention-scale matrices $A=NP$ and $\widehat A=N\widehat P$,
\begin{equation}
\frac{1}{N}\sum_i\left|\sum_j\widehat A_{ij}-1\right|
=
\rho(\widehat P).
\end{equation}
\end{proposition}

\begin{proof}
Row summation is a contraction in $\ell_1$. Combining this fact with the triangle inequality, we have
\begin{equation}
\|\widehat P\mathbf{1}_N-\mu\|_1
\le
\|P\mathbf{1}_N-\mu\|_1
+
\|(\widehat P-P)\mathbf{1}_N\|_1,
\qquad
\|(\widehat P-P)\mathbf{1}_N\|_1
\le
\|\widehat P-P\|_1.
\end{equation}
Pinsker's inequality proves the KL bound. Finally, since $A=NP$, we can write
\begin{equation}
\frac{1}{N}\sum_i\left|\sum_j\widehat A_{ij}-1\right|
=
\frac{1}{N}\sum_i\left|N\sum_j\widehat P_{ij}-1\right|
=
\sum_i\left|\sum_j\widehat P_{ij}-\frac{1}{N}\right|
=
\rho(\widehat P).
\end{equation}
\end{proof}

\begin{proposition}
\label{prop:attention_output_stability}
Under the assumption of Proposition~\ref{prop:dual_stability}, the multiplicative plan bound also controls one attention head. Let
\begin{equation}
y_i = N\sum_j P_{f,ij}v_j,
\qquad
\widehat y_i = N\sum_j P_{\widehat f,ij}v_j,
\end{equation}
and assume $\|v_j\|_2 \le B$ for all $j$. Under this assumption, we have
\begin{equation}
\|\widehat y_i-y_i\|_2
\le
NB\sum_j |P_{\widehat f,ij}-P_{f,ij}|
\le
NB\bigl(e^{2\alpha/\varepsilon}-1\bigr)\sum_j P_{f,ij}.
\end{equation}
If the reference reconstructed plan has row mass $1/N$, this reduces to
\begin{equation}
\|\widehat y_i-y_i\|_2
\le
B\bigl(e^{2\alpha/\varepsilon}-1\bigr).
\end{equation}
For finite-budget or masked conventions, the first bound applies with the actual row mass of the reference reconstructed plan.
\end{proposition}

\begin{proof}
Fix a query token $i$. The output error is the value-weighted gap between the two attention rows. By the triangle inequality and $\|v_j\|_2\le B$, we have
\begin{equation}
\|\widehat y_i-y_i\|_2
\le
NB\sum_j |P_{\widehat f,ij}-P_{f,ij}|.
\end{equation}
The multiplicative plan bound in Proposition~\ref{prop:dual_stability} also implies
\begin{equation}
|P_{\widehat f,ij}-P_{f,ij}|
\le
\bigl(e^{2\alpha/\varepsilon}-1\bigr)P_{f,ij}.
\end{equation}
Using this pointwise bound inside the row sum proves the second displayed inequality. If the reference row mass is $1/N$, then $N\sum_j P_{f,ij}=1$, which yields the simplified bound.
\end{proof}

The same output bound applies to the default ASAP plan after replacing Proposition~\ref{prop:dual_stability} by Proposition~\ref{prop:two_sided_stability}.

\medskip\noindent\textbf{Permutation equivariance.}

\begin{proposition}
\label{prop:permutation_equivariance}
Consider the unmasked square setting with uniform marginals, and let $\Pi$ be a permutation matrix. Assume slice sorting uses a deterministic tie rule that is consistent under token permutation. With this tie rule, the sliced feature matrix satisfies
\begin{equation}
X_\Theta(\Pi Q,\Pi K)=\Pi X_\Theta(Q,K).
\end{equation}
Consequently, ASAP satisfies
\begin{equation}
\widehat f_\omega(\Pi Q,\Pi K)=\Pi\widehat f_\omega(Q,K),
\qquad
\widehat P_\omega(\Pi Q,\Pi K)=\Pi\widehat P_\omega(Q,K)\Pi^\top,
\end{equation}
and, for values permuted in the same way,
\begin{equation}
\operatorname{ASAP}(\Pi Q,\Pi K,\Pi V)
=
\Pi\,\operatorname{ASAP}(Q,K,V).
\end{equation}
\end{proposition}

\begin{proof}
Fix a slice direction. Permuting $Q$ and $K$ permutes the projected source and target values, but the sorted order statistics are unchanged up to the same relabeling of token identities. Under the stated tie rule, the sliced source potential assigned back to token positions satisfies $\tilde f^{(\ell)}(\Pi Q,\Pi K)=\Pi \tilde f^{(\ell)}(Q,K)$ for every slice $\ell$. Thus each sliced feature column permutes in the same way as the tokens, and $X_\Theta(\Pi Q,\Pi K)=\Pi X_\Theta(Q,K)$.

Centering commutes with permutation because $\Pi\mathbf{1}_N=\mathbf{1}_N$, so $\widehat f_\omega(\Pi Q,\Pi K)=\Pi\widehat f_\omega(Q,K)$. The cost matrix transforms as $C(\Pi Q,\Pi K)=\Pi C(Q,K)\Pi^\top$. Since the marginals are uniform in the square setting, the key-side and source-side $c$-transforms both commute with simultaneous token permutation. The one-sided ASAP-0 plan and the default two-sided ASAP plan therefore permute in the same way:
\begin{equation}
\widehat P_\omega(\Pi Q,\Pi K)
=
\Pi\widehat P_\omega(Q,K)\Pi^\top.
\end{equation}
With $A=NP$, the value output is therefore
\begin{equation}
N\widehat P_\omega(\Pi Q,\Pi K)\Pi V
=
N\Pi\widehat P_\omega(Q,K)V
=
\Pi\,\operatorname{ASAP}(Q,K,V).
\end{equation}
\end{proof}

\section{Additional Runtime Wall-Clock Analysis and Patch-Size Impact}
\label{app:runtime}

We keep the main paper focused on trained-layer replacement and place the longer runtime sweep here as supporting inference context.

\subsection{FlashSinkhorn Discussion}

ASAP and FlashSinkhorn reduce two parts of the Sinkhorn attention cost. FlashSinkhorn accelerates the same iterative Sinkhorn computation with IO-aware fused kernels, while ASAP changes the online operator after a teacher has been trained, replacing the iterative scaling loop of that fixed layer by sliced-feature extraction, one source-dual prediction, and the one-sided or two-sided entropic $c$-transform.

The runtime numbers should be read as operator-replacement results under the released PyTorch-style Sinkformer implementations used in our experiments, not as kernel-level optimality claims. A fused Sinkhorn implementation strengthens the iterative baseline, and a fused ASAP implementation could also reduce constants in slice extraction and the $c$-transform. Table~\ref{tab:flashsinkhorn_control} shows the expected memory advantage of the fused iterative implementation because the current ASAP implementation materializes the dense attention matrix.

For teacher fidelity, the relevant comparison is the frozen-layer tradeoff. Low-budget Sinkhorn spends fewer online updates, while ASAP spends one offline fit and replaces the iterative online scaling loop for the same fitted Sinkhorn parameterization.

\subsection{Fused Sinkhorn Control}

We use Table~\ref{tab:flashsinkhorn_control} as a fused-kernel control on matched synthetic $Q,K,V$ tensors. The public FlashSinkhorn kernels implement squared-Euclidean entropic OT, so this control uses unit-normalized queries and keys with the corresponding squared-Euclidean convention. The PyTorch rows reproduce a dense iterative reference implementation, the FlashSinkhorn rows use the released fused iterative kernels, and the ASAP rows report fitted one-sided and two-sided operators with no online Sinkhorn scaling. Because we use this synthetic setting only to compare implementations, the table reports latency and peak memory. Teacher-approximation errors are evaluated on trained-layer activations in the frozen replacement experiments.

\begin{table}[!htbp]
  \centering
  \scriptsize
  \setlength{\tabcolsep}{4pt}
  \resizebox{\linewidth}{!}{%
  \begin{tabular}{llcccc}
    \toprule
    Method & Implementation & $N=512$ ms $\downarrow$ & $N=2048$ ms $\downarrow$ & $N=4096$ ms $\downarrow$ & $N=4096$ GB $\downarrow$ \\
    \midrule
    Sinkhorn normalizer ($S=3$) & PyTorch iterative & 0.46 & 0.75 & 2.88 & 0.313 \\
    Sinkhorn teacher ($S=20$) & PyTorch iterative & 2.00 & 3.54 & 12.86 & 0.313 \\
    FlashSinkhorn ($S=3$) & Fused iterative & 0.37 & $\mathbf{0.44}$ & $\mathbf{0.68}$ & $\mathbf{0.002}$ \\
    FlashSinkhorn ($S=20$) & Fused iterative & 1.16 & 1.69 & 3.13 & $\mathbf{0.002}$ \\
    ASAP-LS-0 ($L=8$) & PyTorch ASAP-0 & 0.33 & 0.71 & 1.80 & 0.313 \\
    ASAP-LS-0 ($L=16$) & PyTorch ASAP-0 & $\mathbf{0.32}$ & 0.71 & 1.83 & 0.313 \\
    ASAP-LS-0 ($L=32$) & PyTorch ASAP-0 & 0.33 & 0.72 & 1.86 & 0.313 \\
    ASAP-LS ($L=32$) & PyTorch ASAP & 0.41 & 0.88 & 2.44 & 0.313 \\
    \bottomrule
  \end{tabular}
  }
  \caption{Fused-Sinkhorn control under matched tensor shapes. Latency and peak memory are measured after CUDA warmup and explicit synchronization.}
  \label{tab:flashsinkhorn_control}
\end{table}

At $N=4096$, FlashSinkhorn cuts iterative memory from $0.313$ GB to $0.002$ GB and reduces the $S=20$ pass from $12.86$ ms to $3.13$ ms. ASAP-0 runs in $1.80$--$1.86$ ms for $L\in\{8,16,32\}$, and the default two-sided ASAP row runs in $2.44$ ms. The fused $S=3$ implementation remains faster and more memory efficient. The offline ASAP fit stays below $0.15$ seconds across the measured lengths and slice counts.

Fused Sinkhorn kernels are complementary to ASAP because they reduce the cost of producing finite-budget teacher potentials during calibration, while the fitted ASAP layer remains a cheaper online operator than the fused high-budget teacher for repeated inference. The memory column reflects the current implementations. FlashSinkhorn streams the iterative computation, while the PyTorch ASAP implementation still materializes dense intermediate matrices.

Although this control clarifies the implementations, it is not the teacher-fidelity evidence. In this normalized squared-Euclidean setting, low- and high-budget Sinkhorn are already close, so trained activations remain the relevant test for whether ASAP follows the trained finite operator.

\subsection{Full Runtime Wall-Clock}

The sweep compares forward operators, not the cost required to obtain each trained model. Under this configuration, ESPFormer (Soft) runs out of memory at $N=8000$, and ESPFormer (Hard) should be read as the post-training hard-sort inference operator reached after soft-sort training and temperature annealing.

\begin{figure}[!htbp]
  \centering
  \includegraphics[width=\textwidth]{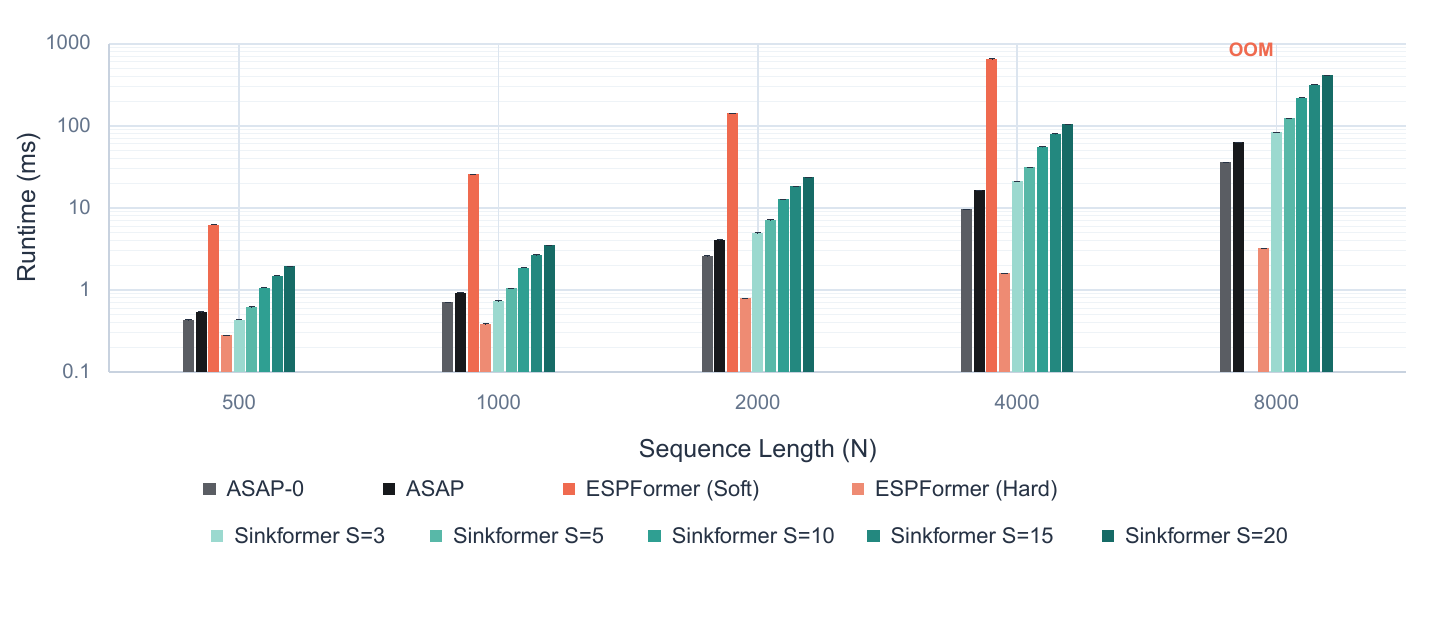}
  \caption{Wall-clock runtime sweep for ASAP, ESPFormer, and Sinkformer operators on synthetic dense self-attention inputs. ASAP-0 and ASAP latencies exclude offline teacher extraction and fitting; speedups are relative to Sinkhorn $S=3$ and $S=20$ under the same benchmark. ESPFormer uses $32$ slices, and the missing ESPFormer (Soft) bar indicates an out-of-memory run.}
  \label{fig:runtime_appendix}
\end{figure}

\begin{table}[!htbp]
  \centering
  \footnotesize
  \setlength{\tabcolsep}{4pt}
  \begin{tabular}{rcccccc}
    \toprule
    $N$ & ASAP-LS-0 & ASAP-LS & Sinkhorn normalizer ($S=3$) & Speedup & Sinkhorn teacher ($S=20$) & Speedup \\
    \midrule
    500 & 0.42 & 0.53 & 0.43 & $1.02\times$ & 1.94 & $4.57\times$ \\
    1000 & 0.70 & 0.92 & 0.73 & $1.04\times$ & 3.49 & $4.96\times$ \\
    2000 & 2.61 & 4.06 & 4.97 & $1.90\times$ & 23.71 & $9.08\times$ \\
    4000 & 9.57 & 16.54 & 21.28 & $2.22\times$ & 104.72 & $10.95\times$ \\
    8000 & 35.70 & 63.51 & 84.03 & $2.35\times$ & 415.80 & $11.65\times$ \\
    \bottomrule
  \end{tabular}
  \caption{Forward latency in milliseconds for ASAP-0, ASAP, the Sinkhorn normalizer ($S=3$), and the Sinkhorn teacher ($S=20$) used for ASAP compilation in Figure~\ref{fig:runtime_appendix}. Ratios are computed as baseline latency divided by ASAP-0 latency from the mean over $10$ timed CUDA runs.}
  \label{tab:runtime_speedup}
\end{table}

\subsection{ESPFormer Training Resources}
\label{app:esp_resource}

ESPFormer trains through soft sorting before switching to hard-sort inference. The hard-sort inference operator is efficient, but the soft-sort training operator builds soft permutation tensors for queries and keys with shape $B\times H\times L\times N\times N$. We measure this operator because it is the training-time operator used before switching to hard sorting. The measurement follows the ESPFormer complexity discussion, where SoftSort and per-slice plan construction contribute $O(LN^2)$ terms, and the ESPFormer limitations note that training memory scales with the number of slices since each slice permutation matrix must be retained.

Figure~\ref{fig:espformer_training_resources} plots the full IMDb encoder profile and a slice-count scaling profile at fixed $B=4$, $N=197$, and $H=8$. Table~\ref{tab:imdb_resource_budget_sweep} reports the full IMDb resource sweep used for the left panel of Table~\ref{tab:operator_main}. The measurement uses real IMDb token batches at sequence length $512$ and effective batch size $32$. ESPFormer (Soft) uses micro-batch $8$ with four gradient-accumulation steps, so all train rows have the same effective batch size.

\begin{figure}[!htbp]
  \centering
  \includegraphics[width=0.90\textwidth]{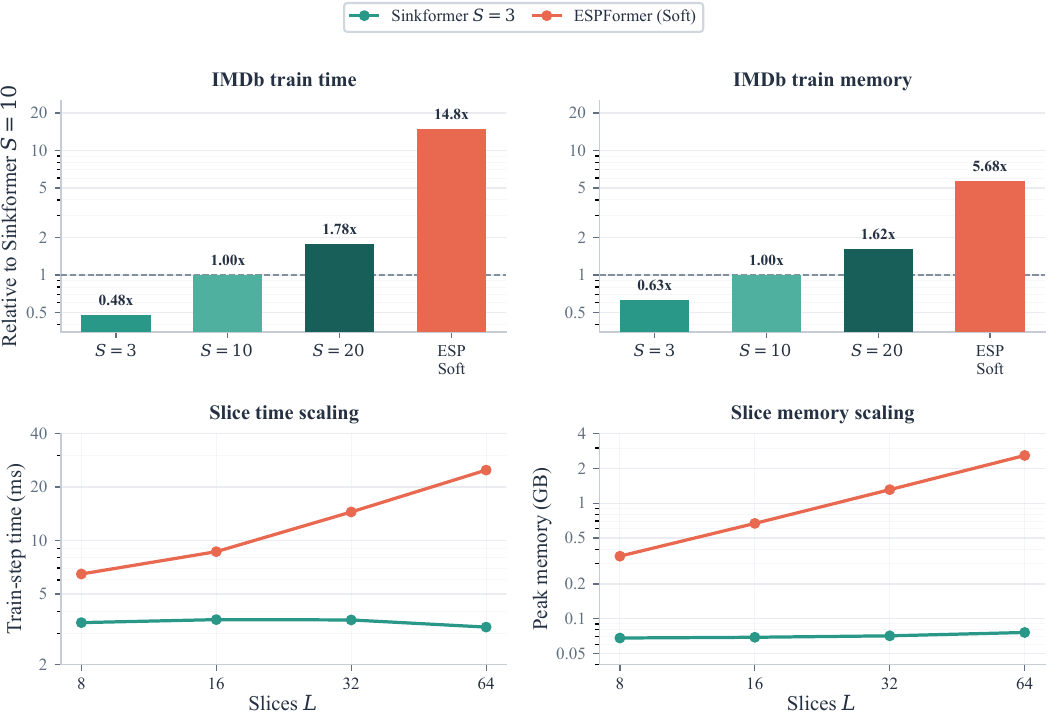}
  \caption{ESPFormer training-resource profiles. Top panels use full IMDb encoder batches at sequence length $512$, normalized to Sinkformer $S=10$. Bottom panels show slice-count scaling at fixed $B=4$, $N=197$, and $H=8$.}
  \label{fig:espformer_training_resources}
\end{figure}

\begin{table}[!htbp]
  \centering
  \scriptsize
  \setlength{\tabcolsep}{3pt}
  \begin{tabular}{lrrrrrr}
    \toprule
    Method & Train ms/ex $\downarrow$ & Train GB $\downarrow$ & Train time vs.\ $S=10$ & Train mem.\ vs.\ $S=10$ & Forward ms/ex $\downarrow$ & Forward GB $\downarrow$ \\
    \midrule
    Sinkformer ($S=3$) & 4.49 & 7.72 & 0.48 & 0.63 & 2.39 & 1.12 \\
    Sinkformer ($S=10$) & 9.37 & 12.24 & 1.00 & 1.00 & 5.75 & 1.12 \\
    Sinkformer ($S=20$) & 16.71 & 19.77 & 1.78 & 1.62 & 10.58 & 1.12 \\
    ESPFormer (Soft) & 138.82 & 69.51 & 14.82 & 5.68 & 47.92 & 8.58 \\
    \bottomrule
  \end{tabular}
  \caption{Full-encoder IMDb resource profile on RTX PRO 6000 Blackwell Server Edition GPUs. Train rows use effective batch size $32$; forward rows report evaluation batches from the same sequence-length setting. Ratios are normalized to Sinkformer $S=10$.}
  \label{tab:imdb_resource_budget_sweep}
\end{table}

Table~\ref{tab:catsdogs_resource_h100} uses random Cats and Dogs ViT-shaped inputs to measure the attention-operator resource cost at the same batch size and image resolution as the downstream protocol. In this GPU measurement, ESPFormer (Soft) forward evaluation is $19.8\times$ slower and uses $43.1\times$ more peak GPU memory than Sinkformer ($S=10$). ESPFormer (Hard) is slower still in this measurement because it constructs the soft-sort tensors before hardening. The ESPFormer (Soft) train step runs out of memory at batch size $64$ on the same runtime, while Sinkformer train steps at $S\in\{3,10,20\}$ complete.

\begin{table}[!htbp]
  \centering
  \scriptsize
  \setlength{\tabcolsep}{4pt}
  \begin{tabular}{lrrrr}
    \toprule
    Method & Forward ms $\downarrow$ & Forward GB $\downarrow$ & Train-step ms $\downarrow$ & Train-step GB $\downarrow$ \\
    \midrule
    Sinkformer ($S=3$) & 16.27 & 0.48 & 34.09 & 3.54 \\
    Sinkformer ($S=10$) & 38.84 & 0.48 & 79.88 & 6.90 \\
    Sinkformer ($S=20$) & 71.15 & 0.48 & 145.23 & 11.71 \\
    ESPFormer (Soft) & 768.78 & 20.72 & OOM & OOM \\
    ESPFormer (Hard) & 1603.83 & 30.82 & -- & -- \\
    \bottomrule
  \end{tabular}
  \caption{Full Cats and Dogs ViT-shape resource profile on random inputs of shape $64\times3\times224\times224$. Entries report mean wall-clock time after warmup and peak total allocated GPU memory. ESPFormer (Soft) uses the released soft-sort implementation, and ESPFormer (Hard) is a forward-only measurement because training uses the soft-sort row.}
  \label{tab:catsdogs_resource_h100}
\end{table}

The ESPFormer training-resource cells in Table~\ref{tab:catsdogs_main} come from the released soft-sort training log after dropping the first cold step. Because the run uses multiple devices, the table reports summed peak memory.

Table~\ref{tab:esp_training_sequence} measures tensor-level train-step scaling with sequence length at batch size one and $64$ slices. The soft-sort training operator retains dense per-slice sorting tensors, while Sinkhorn scaling keeps the online state closer to the attention matrix size.

\begin{table}[!htbp]
  \centering
  \scriptsize
  \setlength{\tabcolsep}{3pt}
  \resizebox{\linewidth}{!}{%
  \begin{tabular}{rcccccccc}
    \toprule
    & \multicolumn{2}{c}{Sinkformer ($S=3$)} & \multicolumn{2}{c}{Sinkformer ($S=10$)} & \multicolumn{2}{c}{ESPFormer (Soft)} & \multicolumn{2}{c}{ESPFormer (Hard)} \\
    \cmidrule(lr){2-3}\cmidrule(lr){4-5}\cmidrule(lr){6-7}\cmidrule(lr){8-9}
    $N$ & ms & GB & ms & GB & ms & GB & ms & GB \\
    \midrule
    197 & 4.32 & 0.03 & 5.94 & 0.05 & 8.66 & 0.66 & 13.55 & 0.82 \\
    512 & 3.78 & 0.11 & 7.58 & 0.22 & 47.83 & 4.34 & 71.86 & 5.41 \\
    1000 & 5.09 & 0.35 & 12.59 & 0.80 & 229.42 & 16.49 & 292.00 & 20.59 \\
    2000 & 12.15 & 1.32 & 34.78 & 3.11 & 2216.72 & 65.85 & 2700.73 & 82.23 \\
    \bottomrule
  \end{tabular}
  }
  \caption{Train-step scaling with sequence length on A100 GPUs, using batch size $1$, $8$ heads, and $64$ slices. Each entry reports wall-clock milliseconds and peak allocated GPU memory in GB.}
  \label{tab:esp_training_sequence}
\end{table}

\subsection{ESPFormer Soft-to-Hard Switch}
\label{app:esp_soft_hard_switch}

ESPFormer uses soft sorting during training and reports hard sorting as an efficient inference operator after temperature annealing \citep{shahbazi2025espformer}. A hard-sort gain is plausible when annealing has already moved the trained weights close to the hard-sorting regime. In that case, the soft operator behaves like a noisy relaxation, and hard permutations can sharpen the intended sliced transport plan. On an easier binary task such as Cats and Dogs, the classifier can also tolerate a larger logit change if the decision boundary remains stable.

The same switch can substantially change the evaluated operator. Hard sorting is a discontinuous limit of the soft relaxation and produces much lower-entropy attention. To quantify the effect, we evaluate the same CIFAR-100 ESPFormer checkpoint under three operators: the soft training relaxation, the released SoftSort hard flag, and the exact argsort operator that sorts each query and key slice, matches equal ranks, and averages the slice plans. The released hard flag hardens the soft relaxation by top-one selection, but it is not the exact permutation operator. The exact argsort operator is the relevant hard-sort reference for marginal balance.

\begin{table}[!htbp]
  \centering
  \scriptsize
  \setlength{\tabcolsep}{4pt}
  \resizebox{\linewidth}{!}{%
  \begin{tabular}{lcccccc}
    \toprule
    Method & Test acc. (\%) & Pred. change (\%) & Row err. & Col err. & Entropy & Val ex./s \\
    \midrule
    ESPFormer (Soft) & $41.89$ & $0.00$ & $0.059$ & $0.060$ & $2.858$ & $948$ \\
    ESPFormer (Released Hard) & $35.31$ & $41.02$ & $0.148$ & $0.115$ & $0.176$ & $451$ \\
    ESPFormer (Exact Hard) & $36.75$ & $34.65$ & $5.9{\times}10^{-4}$ & $5.9{\times}10^{-4}$ & $0.186$ & $6572$ \\
    \bottomrule
  \end{tabular}
  }
  \caption{CIFAR-100 soft-to-hard comparison for the same epoch-50 ESPFormer checkpoint. Test accuracy and prediction change are measured on the test set, with prediction change relative to ESPFormer (Soft). Row and column errors are mean absolute deviations from one, averaged over recorded attention layers and batches. Throughput is measured on the validation split with batch size $256$.}
  \label{tab:esp_soft_hard_switch}
\end{table}

Table~\ref{tab:esp_soft_hard_switch} compares the two hard-sort variants. ESPFormer (Exact Hard) restores the row and column marginals and is fastest in this run. Its task behavior still does not match ESPFormer (Soft), with test accuracy dropping from $41.89$ to $36.75$ and predictions changing on $34.65\%$ of test examples. ESPFormer (Released Hard) is less balanced and lower in accuracy than ESPFormer (Exact Hard). We do not read this as a contradiction of the published Cats and Dogs hard-sort row; instead, it shows that hard-sort ESP inference is another operator whose task effect should be measured, not a cheaper evaluation of the same trained soft-sort operator. ASAP avoids that soft-to-hard operator switch by training the Sinkhorn operator directly and compiling the finite Sinkhorn duals after teacher training.

\subsection{Patch-Size Impact under Sinkhorn Layer Replacement}
\label{app:mnist_replacement}

The shallow MNIST experiment follows the one-layer ViT protocol used by Sinkformers and ESPFormer, where each attention mechanism is trained end to end across patch sizes. We use the same setting for replacement by training a Sinkhorn teacher ($S=5$), freezing its learned projections and classifier, and replacing only the attention normalization at evaluation time. We report the stable, nontrivial replacement settings, patch sizes $2,4,7,14$. ASAP-0 is the one-sided key-side ablation, and ASAP is the row-ending two-sided transform for this $S=5$ teacher.

\begin{table}[!htbp]
  \centering
  \scriptsize
  \setlength{\tabcolsep}{3pt}
  \resizebox{\linewidth}{!}{%
  \begin{tabular}{lcccc}
    \toprule
    Method & Patch 2 & Patch 4 & Patch 7 & Patch 14 \\
    \midrule
    Sinkhorn teacher ($S=5$) & $96.87 \pm 0.18$ / $100.00$ & $96.89 \pm 0.52$ / $100.00$ & $96.27 \pm 0.30$ / $100.00$ & $96.32 \pm 0.12$ / $100.00$ \\
    \midrule
    Softmax replacement & $24.20 \pm 12.37$ / $24.34 \pm 12.53$ & $10.77 \pm 1.60$ / $10.76 \pm 1.73$ & $13.89 \pm 1.36$ / $13.88 \pm 1.44$ & $19.06 \pm 5.75$ / $19.07 \pm 5.72$ \\
    Sinkhorn normalizer ($S=3$) & $89.51 \pm 2.66$ / $90.22 \pm 2.56$ & $90.75 \pm 4.45$ / $91.77 \pm 4.78$ & $76.86 \pm 11.26$ / $77.54 \pm 11.50$ & $92.94 \pm 3.59$ / $94.31 \pm 3.86$ \\
    ASAP-0 & $94.74 \pm 0.06$ / $95.70 \pm 0.12$ & $94.61 \pm 1.40$ / $96.18 \pm 1.31$ & $88.76 \pm 9.06$ / $90.12 \pm 9.22$ & $95.14 \pm 1.34$ / $97.45 \pm 1.57$ \\
    ASAP & $96.03 \pm 0.18$ / $97.28 \pm 0.10$ & $96.19 \pm 0.82$ / $97.96 \pm 0.80$ & $95.18 \pm 1.30$ / $97.27 \pm 1.50$ & $96.29 \pm 0.13$ / $99.25 \pm 0.29$ \\
    \bottomrule
  \end{tabular}
  }
  \caption{Patch-size impact under Sinkhorn layer replacement. Each entry is accuracy / teacher agreement in percent, reported as mean $\pm$ standard deviation over three runs. Teacher agreement is prediction agreement with the frozen Sinkhorn teacher ($S=5$).}
  \label{tab:mnist_replacement}
\end{table}

With one attention layer and no feed-forward block, this is a sensitive replacement test with little redundancy for normalization mismatch. On patch size $2$, ASAP-0 substantially improves over the Sinkhorn normalizer ($S=3$) replacement, reaching $94.74\%$ accuracy and $95.70\%$ teacher agreement. ASAP recovers most teacher behavior, reaching $96.03\%$ accuracy and $97.28\%$ agreement; after five post-replacement epochs on patch size $2$, ASAP reaches $96.87 \pm 0.15\%$ accuracy, matching the Sinkhorn teacher continuation within run-to-run variation.

Across patch sizes $4,7,14$, ASAP stays close to the teacher, with mean accuracy drops of $0.71$, $1.10$, and $0.02$ percentage points, respectively, while maintaining substantially higher teacher agreement than the Sinkhorn normalizer ($S=3$) replacement. This sensitive setting shows the value of ASAP when a low-redundancy architecture needs tighter normalization fidelity.

\section{Implementation Details}
\label{app:details}

In Algorithm~\ref{alg:asap}, \textsc{TeacherSource} returns the finite source dual used before the teacher's final closure under its update convention. For column-ending teachers, the final key-side transform of this source dual recovers the teacher plan under exact prediction. The default ASAP branch then applies further finite closures from the fitted source dual.

\begin{algorithm}[!htbp]
\caption{ASAP offline fit and online layer}
\label{alg:asap}
\centering
\begin{minipage}[t]{0.48\textwidth}
\footnotesize
\textbf{Offline sliced-dual projection}
\vspace{2pt}
\begin{algorithmic}[1]
\Require Fit set $\mathcal{D}$, slices $\Theta$, teacher budget $T$, entropy $\varepsilon$, ridge $\lambda$, objective
\State $G \gets \lambda I_L$, $b \gets 0$
\For{$m=1,\ldots,M$}
    \State $X_m \gets X_\Theta(Q^{(m)},K^{(m)})$
    \State $C^{(m)} \gets \textsc{Cost}(Q^{(m)},K^{(m)})$
    \State $f^{(T,m)} \gets \textsc{TeacherSource}_T(Q^{(m)},K^{(m)})$
    \State $\tilde f^{(T,m)} \gets \Pi_0(f^{(T,m)})$
    \If{objective is LS}
        \State $G \gets G + X_m^\top X_m$
        \State $b \gets b + X_m^\top \tilde f^{(T,m)}$
    \Else
        \State $P^{(T,m)} \gets \mathcal{C}_\nu(\tilde f^{(T,m)};C^{(m)})$
        \State store $(X_m,C^{(m)},P^{(T,m)})$ for Eq.~\ref{eq:asap_kl}
    \EndIf
\EndFor
\If{objective is LS}
    \State $\omega \gets G^{-1}b$
\Else
    \State $\omega \gets$ minimizer of Eq.~\ref{eq:asap_kl}
\EndIf
\State \Return $\omega$
\end{algorithmic}
\end{minipage}
\hfill
\begin{minipage}[t]{0.48\textwidth}
\footnotesize
\textbf{Online compiled attention}
\vspace{2pt}
\begin{algorithmic}[1]
\Require Evaluation pair $(Q,K)$, values $V$, slices $\Theta$, coefficient vector $\omega$, teacher last step $s\in\{\mathrm{row},\mathrm{column}\}$
\State $X \gets X_\Theta(Q,K)$
\State $C \gets \textsc{Cost}(Q,K)$
\State $\hat f \gets \Pi_0(X\omega)$
\If{ASAP-0}
    \State $\hat g^0 \gets T_K^\nu(\hat f)$
    \State $\hat P_{ij} \gets \exp\!\left(\frac{-C_{ij}+\hat f_i+\hat g^0_j}{\varepsilon}\right)$
\ElsIf{ASAP}
    \State $\hat g^0 \gets T_K^\nu(\hat f)$
    \State $\hat f^+ \gets T_Q^\mu(\hat g^0)$
    \If{$s=\mathrm{column}$}
        \State $\hat g^+ \gets T_K^\nu(\hat f^+)$
        \State $\hat P_{ij} \gets \exp\!\left(\frac{-C_{ij}+\hat f^+_i+\hat g^+_j}{\varepsilon}\right)$
    \Else
        \State $\hat P_{ij} \gets \exp\!\left(\frac{-C_{ij}+\hat f^+_i+\hat g^0_j}{\varepsilon}\right)$
    \EndIf
\EndIf
\State $\hat A \gets N\hat P$
\State \Return $\hat A V$
\end{algorithmic}
\end{minipage}
\end{algorithm}

This section fixes the implementation choices that couple the frozen Sinkhorn teacher to the compiled ASAP layer.

\subsection{Teacher Attention and Offline Fit}

ASAP is fit to the finite-budget teacher used by the trained model, not an exact-convergence idealization. For each task, we use the teacher's entropy, iteration budget, mask, and update convention when extracting duals or plans. ASAP-LS centers the source log-dual and solves the sliced-potential regressor by closed-form ridge fitting. ASAP-KL uses the same sliced features and source-dual parameterization, but fits the coefficients through the teacher plan after the key-side $c$-transform.

The experiments in this paper use a short post-training fitting pass after the teacher is frozen. The same LS normal equations may also be accumulated during a late training phase, but if the teacher parameters change during accumulation the fit corresponds to the sequence of observed teacher states, not exactly to the final frozen teacher.

We use the score-to-cost conversion because the teacher is implemented in score-kernel coordinates, while sliced potentials are quadratic-cost features. When the teacher backbone is defined from scaled dot-product scores, teacher extraction and inference-time reconstruction use that same finite score-kernel convention. The score-kernel source dual is first shifted into quadratic-cost coordinates by adding $\rho_i(Q)=\frac{1}{2\sqrt{d_h}}\|q_i\|_2^2$ before fitting. At inference time the predicted quadratic-cost source dual is shifted back by subtracting $\rho_i(Q)$ before the score-kernel $c$-transforms. The conversion is a finite-coordinate conversion for the implemented score-kernel teacher. It is not a claim that finite score-kernel Sinkhorn and finite quadratic-cost Sinkhorn from zero initialization have identical iterates. In particular, the sentiment experiments follow the released Sinkformer text score-kernel implementation while applying this dual-coordinate conversion.

Fit examples are drawn from the same unlabeled in-distribution training pipeline as the teacher and are processed with the same tokenizer or image preprocessing. Unless noted otherwise, the fit is performed per attention layer, with token positions and head instances pooled within that layer, and the same slice set $\Theta$ is reused across all fit and evaluation examples for a given compiled layer. In plug-and-play evaluation, the compiled ASAP layer is used after fitting; when a post-replacement refresh updates upstream query and key projections, gradients through hard sorting are treated as the implementation's almost-everywhere sort gradients unless a given experiment explicitly freezes those features.

\paragraph{Masked text attention.}
\label{app:masked_text}
For padded text batches, let $J(x)\subseteq\{1,\ldots,N\}$ denote the active key positions for an input $x$. The released sentiment implementation masks padded key columns but keeps all $N$ query positions in the sequence. We match this convention for teacher extraction, ASAP fitting, and ASAP reconstruction. In coupling units, the masked key target is
\begin{equation}
\nu^{J}_j =
\begin{cases}
1/N, & j\in J(x),\\
0, & j\notin J(x).
\end{cases}
\end{equation}
For $j\in J(x)$, the key-side $c$-transform is
\begin{equation}
\hat g_{\omega,j}
=
\varepsilon\log(1/N)
-
\varepsilon\log\sum_{i=1}^N
\exp\!\left(\frac{-C_{ij}+\hat f_{\omega,i}}{\varepsilon}\right),
\end{equation}
and for $j\notin J(x)$ we set $\hat g_{\omega,j}=-\infty$ and $\hat P_{\omega,ij}=0$. With the attention-scale matrix $\hat A_\omega=N\hat P_\omega$, each active key column has target sum one and padded columns have target sum zero.

When $|J(x)|<N$, the masked object has total mass $|J(x)|/N$, so the text experiments compile the finite key-masked Sinkformer convention, not the square balanced transport ideal used in the main text. The balanced positive-marginal propositions in Appendix~\ref{app:proofs} apply to the unmasked square setting. For masked sliced features, sorting is applied on active token positions only, and padded positions do not enter the one-dimensional key support or the regression loss. Row errors in the sentiment experiments are reported in their own table and reflect the active-key mask and final normalization convention as well as source-dual approximation error. Column errors are averaged only over active key columns.

\begin{proposition}
\label{prop:masked_key_transform}
Let $J\subseteq\{1,\ldots,N\}$ be the active key set and define the masked key target
\begin{equation}
\nu^J_j =
\begin{cases}
1/N, & j\in J,\\
0, & j\notin J.
\end{cases}
\end{equation}
For any source dual $f\in\mathbb{R}^N$, define, for $j\in J$,
\begin{equation}
g^J_j
=
\varepsilon\log(1/N)
-
\varepsilon\log
\sum_{i=1}^N
\exp\left(\frac{-C_{ij}+f_i}{\varepsilon}\right),
\end{equation}
and set $g^J_j=-\infty$ for $j\notin J$. With the convention $\exp(-\infty)=0$, let
\begin{equation}
P^J_{ij}
=
\exp\left(\frac{-C_{ij}+f_i+g^J_j}{\varepsilon}\right).
\end{equation}
With this convention, we have
\begin{equation}
\sum_i P^J_{ij}=1/N \quad \text{for } j\in J,
\qquad
P^J_{ij}=0 \quad \text{for } j\notin J.
\end{equation}
The total mass of $P^J$ is $|J|/N$.
\end{proposition}

\begin{proof}
For an active key $j\in J$, exponentiating the definition of $g^J_j$ yields
\begin{equation}
\exp\left(\frac{g^J_j}{\varepsilon}\right)
=
\frac{1/N}{\sum_i \exp\left(\frac{-C_{ij}+f_i}{\varepsilon}\right)}.
\end{equation}
For an active key, the same cancellation as in the unmasked case leaves exactly the active-key mass:
\begin{equation}
\sum_i P^J_{ij}
=
1/N.
\end{equation}
For an inactive key $j\notin J$, the convention $\exp(-\infty)=0$ forces $P^J_{ij}=0$ for all $i$. Summing the active column masses, the total mass is $|J|/N$.
\end{proof}

\subsection{Online and Offline Cost}
\label{app:cost_breakdown}

For one dense forward pass, ASAP requires $O(N^2d_h)$ work to form the dense score or cost matrix, $O(LNd_h)$ for slice projections, $O(LN\log N)$ for slice-wise sorting and potential extraction, $O(NL)$ for dual prediction, $O(N^2)$ work per entropic $c$-transform call, and $O(N^2 d_v)$ for multiplying the dense attention matrix by the value matrix. ASAP-0 uses one $c$-transform call. Column-ending ASAP uses the three calls in Equation~\ref{eq:two_sided_c_transform}; row-ending ASAP uses two calls. In implementation, the key-side transform followed by multiplication by $N$ is evaluated in its algebraically equivalent column-wise softmax form. In the unmasked square ASAP-0 case,
\begin{equation}
\hat A_{\omega,ij}
=
\frac{
\exp\!\left(( -C_{ij}+\hat f_{\omega,i})/\varepsilon\right)
}{
\sum_{i'=1}^N
\exp\!\left(( -C_{i'j}+\hat f_{\omega,i'})/\varepsilon\right)
}.
\end{equation}
The dense online cost is $O(N^2d_h + LNd_h + LN\log N + NL + R N^2 + N^2d_v)$ with $R=1$ for ASAP-0, $R=3$ for column-ending ASAP, and $R=2$ for row-ending ASAP, compared with $O(N^2d_h + SN^2 + N^2d_v)$ for $S$ Sinkhorn scaling steps after score formation. For $M$ fit examples, the offline work includes teacher extraction, typically $O(MTN^2)$ for $T$ teacher sweeps after costs are formed, sliced-feature computation, and coefficient fitting. ASAP-LS adds $O(MNL^2)$ statistic accumulation and an $O(L^3)$ solve per fitted layer. ASAP-KL adds a low-dimensional convex optimization over the same $L$ coefficients, where each objective or gradient pass evaluates the key-normalized Gibbs plan on the calibration batch. These offline terms are not part of the online forward latency reported in the runtime tables.

The offline fit pays off after repeated evaluations of the same frozen layer. In the frozen IMDb sweep, ASAP-LS with $L=32$ and $M=4096$ uses $102.0$ seconds of fit time. The default two-sided row saves $63.02-11.79=51.23$ ms per layer-batch relative to the $S=20$ teacher, for a break-even point of approximately $2.0\times10^3$ layer-batch evaluations. Relative to Sinkformer ($S=3$), the saving is $14.14-11.79=2.35$ ms, for approximately $4.3\times10^4$ layer-batch evaluations. The faster ASAP-LS-0 ablation uses the same fit time and saves $63.02-7.74=55.28$ ms against the $S=20$ teacher and $14.14-7.74=6.40$ ms against Sinkformer ($S=3$), for break-even points of approximately $1.8\times10^3$ and $1.6\times10^4$ layer-batch evaluations, respectively. These counts use forward latency only and quantify the repeated-inference regime targeted by post-training fitting.

\subsection{Timing and Evaluation Protocol}

The runtime sweep in Figure~\ref{fig:runtime_appendix} was executed on NVIDIA RTX PRO 6000 Blackwell Server Edition GPU hardware with PyTorch $2.10.0{+}$cu128. We use batch size $1$, sequence lengths $N \in \{500,1000,2000,4000,8000\}$, width $d=1024$, $8$ attention heads, and head dimension $64$. Sinkhorn normalizer baselines are benchmarked at iteration budgets $S \in \{3,5,10,15,20\}$. ESPFormer uses soft-sort temperature $t = 10.0$ and inverse-temperature parameter $\tau = 10^{-3}$. ASAP uses $L=32$ slices, ridge $\lambda = 10^{-3}$, $8$ fit examples, fit batch size $2$, and a Sinkhorn teacher ($S=20$). All timings use $3$ warmup runs, $10$ timed runs, and explicit CUDA synchronization around each measurement. Figure~\ref{fig:runtime_appendix} and Table~\ref{tab:runtime_speedup} report forward latency only; the one-time ASAP fit is reported when fit cost is part of the comparison.

\section{Experiment Details}
\label{app:expdetails}

Table~\ref{tab:exp_protocols} fixes the main experimental settings before the task-specific details.

\begin{table}[!htbp]
  \centering
  \footnotesize
  \setlength{\tabcolsep}{4pt}
  \renewcommand{\arraystretch}{1.08}
  \begin{tabular}{@{}p{0.18\linewidth}p{0.39\linewidth}p{0.37\linewidth}@{}}
    \toprule
    Benchmark & Training and evaluation setting & ASAP-specific setup \\
    \midrule
    IMDb replacement & New matched benchmark with an IMDb encoder at sequence length $512$. Resource profile uses effective batch size $32$ and three trained $S=20$ teachers. & Compile each frozen $S=20$ teacher and report resource use, layer-batch latency, teacher error, marginal error, and frozen accuracy. \\
    Cats and Dogs & ESPFormer reproduction of the Sinkformer setup with a ViT using embed dim $128$, depth $6$, $8$ heads, patch size $16$, Adam, batch size $64$, and $300$ epochs. & Compile after teacher training and report the full-data setting. High-$S$ rows are controlled extensions. \\
    Text classification & Controlled topic and sentiment benchmark set with a depth-$6$ encoder, batch size $32$, $15$ epochs, learning rate $10^{-4}$, and sequence length $128$. & Compile from the $S=10$ teacher using $4096$ unlabeled training inputs. Rows share the same checkpoint rule. \\
    \bottomrule
  \end{tabular}
  \caption{Compact summary of the exact main-paper settings. Where Sinkformer and ESPFormer do not use the same shared-benchmark settings, the main comparison follows the ESPFormer reproduction protocol and the local paragraph below notes the setup mismatch.}
  \label{tab:exp_protocols}
\end{table}

Each paragraph below lists the architecture, optimizer, schedule, and ASAP fit used by the corresponding main-paper result. When Sinkformer and ESPFormer do not use the same shared settings, we follow the ESPFormer reproduction for the main comparison and note the mismatch locally.

\subsection{Operator-Level Benchmarks}

The main frozen-layer replacement benchmark uses trained IMDb encoders with Sinkhorn attention. For each run, we train a Sinkhorn teacher at budget $S=20$, freeze the model, and replace only the attention normalizer at evaluation time. The Sinkhorn normalizer ($S=3$) row reuses the same trained query, key, value, and output projections but runs only three Sinkhorn updates. ASAP rows reuse the same projections and substitute the fitted operator with the one-sided or two-sided $c$-transform. Table~\ref{tab:operator_main} reports layer-batch runtime, teacher-relative output RMSE, attention relative $\ell_2$ error, and final frozen-classifier accuracy over three trained teachers. Table~\ref{tab:operator_marginal_appendix} reports the corresponding row and column marginal errors, and Table~\ref{tab:operator_ablation_appendix} reports the ASAP-LS slice-count sweep on the same held-out batch-layer cases.

The synthetic runtime sweep in Figure~\ref{fig:runtime_appendix} uses matched dense self-attention inputs at sequence lengths $N \in \{500,1000,2000,4000,8000\}$, batch size $1$, model width $d=1024$, $8$ attention heads, and head dimension $64$. Runtime is averaged over $10$ timed runs after $3$ warmup runs with explicit CUDA synchronization. ESPFormer (Soft) and ESPFormer (Hard) use $32$ slices under the shared synthetic width setting. These synthetic rows are used for forward-latency context, not as teacher-agreement evidence for the trained IMDb layer, because ESPFormer constructs another expected sliced transport plan.

\subsection{Cats and Dogs}

For Cats and Dogs, we use the ESPFormer reproduction of the Sinkformer ViT setup. Images are resized to $224 \times 224$ and processed by a ViT with embedding dimension $128$, MLP dimension $128$, depth $6$, $8$ attention heads, and patch size $16$. The ESPFormer reproduction includes $10\%$, $25\%$, and full-data settings; the main table reports the full-data setting. Training uses Adam, batch size $64$, $300$ epochs, and an initial learning rate of $3 \times 10^{-5}$, reduced by a factor of $10$ after $250$ epochs. Normalization and data augmentation follow the Sinkformer protocol. The published comparison rows use three runs, matching the ESPFormer reproduction. This does not exactly match the original Sinkformer appendix, which reports a $5 \times 10^{-5}$ initial learning rate and a five-run summary. For ASAP, we train the Sinkhorn teacher with the same backbone and optimization schedule, then fit the amortized layer once before evaluation.

\subsection{Text Classification}

The text-classification experiments use the same encoder-plus-pooling backbone family as IMDb, with depth $6$, batch size $32$, Adam, $15$ epochs, and hidden width $256$ with $8$ attention heads. We use maximum sequence length $128$ on the standard DBpedia-14, AG News, and Yelp Review Full splits; the tasks have $14$, $4$, and $5$ classes, respectively. The tokenizer assets are shared with the self-contained sentiment notebooks, but all models are trained from scratch on their target dataset with no transfer from IMDb. Sinkhorn teachers use the score-kernel Sinkhorn convention with active-key masking for padded text. The main table reports mean and standard deviation over three runs under a shared best-evaluation checkpoint rule. ASAP is fit from the Sinkhorn teacher ($S=10$) on the first $4096$ training examples, using only inputs and discarding labels.

\subsection{Shallow MNIST Layer Replacement Protocol}
\label{app:mnist_patch_details}

The shallow MNIST layer-replacement experiment follows the one-layer ViT setting used in Sinkformer and ESPFormer to focus on the attention mechanism instead of deeper architectural choices. Images are processed by a single-layer self-attention model with one head and no feed-forward block. Training uses Adam, batch size $100$, and $45$ epochs. As in the reference papers, the Transformer-style baselines use an initial learning rate of $10^{-3}$, whereas the Sinkhorn-style variants use $2 \times 10^{-3}$. In both cases the learning rate is reduced by a factor of $10$ after epochs $35$ and $41$.

ASAP uses $256$ sliced-potential features, a residual log-sum-exp dual fit, and $8192$ unlabeled calibration examples. The $S=5$ teacher ends with a row normalization, so ASAP uses the row-ending two-sided transform. ASAP-0 is the fitted one-sided key-side ablation.

The reported refresh updates the ASAP coefficients, token bias, output projection, layer normalization, and classifier head for five supervised epochs at learning rate $2\times 10^{-4}$. The Sinkformer ($S=5$) continuation row uses the same number of extra epochs as a control. Table~\ref{tab:mnist_replacement} reports patch-size replacement results over three runs.

\section{Additional Ablations}
\label{app:ablations}

\subsection{Further Results for the Frozen-Layer Replacement}

The main operator table measures teacher agreement. The marginal table below checks which side of the Sinkhorn constraint each replacement preserves. For an attention matrix $A$, let $r_i(A)=\sum_j A_{ij}$ and $c_j(A)=\sum_i A_{ij}$. On padded IMDb batches, $\mathcal{J}$ denotes the active non-padded key columns. We report $\Delta_{\mathrm{row}\to 1}=N^{-1}\sum_i |r_i(A)-1|$ and $\Delta_{\mathrm{col}\to 1}=|\mathcal{J}|^{-1}\sum_{j\in\mathcal{J}} |c_j(A)-1|$. These errors depend on the last row or column normalization under key padding. Sinkhorn normalizers at $S=3$ and $S=5$ end with a row normalization, while ASAP and the Sinkhorn teacher ($S=20$) end with a column normalization.

\begin{table}[!htbp]
  \centering
  \scriptsize
  \setlength{\tabcolsep}{3pt}
  \renewcommand{\arraystretch}{0.94}
  \resizebox{\textwidth}{!}{%
  \begin{tabular}{llcc}
    \toprule
    Method & Final update & $\Delta_{\mathrm{row}\to 1}$ $\downarrow$ & $\Delta_{\mathrm{col}\to 1}$ $\downarrow$ \\
    \midrule
    ASAP-LS-0 ($L=32$) & one-sided key & $0.514 \pm 0.053$ & $(1.24 \pm 0.15)\!\times\!10^{-7}$ \\
    ASAP-KL-0 ($L=32$) & one-sided key & $0.512 \pm 0.049$ & $(1.23 \pm 0.15)\!\times\!10^{-7}$ \\
    ASAP-LS ($L=32$) & two-sided column-ending & $0.444 \pm 0.037$ & $(1.27 \pm 0.15)\!\times\!10^{-7}$ \\
    ASAP-KL ($L=32$) & two-sided column-ending & $0.444 \pm 0.037$ & $(1.27 \pm 0.15)\!\times\!10^{-7}$ \\
    Sinkhorn normalizer ($S=3$) & row-ending & $(3.33 \pm 0.23)\!\times\!10^{-7}$ & $0.833 \pm 0.140$ \\
    Sinkhorn normalizer ($S=5$) & row-ending & $(3.13 \pm 0.13)\!\times\!10^{-7}$ & $0.805 \pm 0.127$ \\
    Sinkhorn teacher ($S=20$) & column-ending & $0.441 \pm 0.038$ & $(2.70 \pm 0.23)\!\times\!10^{-7}$ \\
    \bottomrule
  \end{tabular}
  }
\caption{Marginal errors for the same $144$ frozen IMDb batch-layer cases as Table~\ref{tab:operator_main}.}
  \label{tab:operator_marginal_appendix}
\end{table}

\subsection{ASAP-LS Slice Count Ablation}

Table~\ref{tab:operator_ablation_appendix} reports a dedicated ASAP-LS slice-count sweep on the frozen IMDb Sinkhorn teacher ($S=20$) cases. The one-sided ASAP-LS-0 rows show the raw sliced-dual projection tradeoff, while the default two-sided ASAP-LS rows use the same fitted source duals and apply the column-ending two-sided $c$-transform. Increasing $L$ improves teacher agreement for both variants and adds latency through the sliced-feature computation.

\begin{table}[!htbp]
  \centering
  \scriptsize
  \setlength{\tabcolsep}{5pt}
  \begin{tabular}{lccccc}
    \toprule
    Setting & Time (ms) $\downarrow$ & Output RMSE $\downarrow$ & Attention rel.\ $\ell_2$ $\downarrow$ & Fit time (s) $\downarrow$ & Cases \\
    \midrule
    ASAP-LS-0 ($L=8,\ M=4096$) & $\mathbf{5.00}$ & $0.0278$ & $0.483$ & $103.6$ & $144$ \\
    ASAP-LS-0 ($L=16,\ M=4096$) & $5.22$ & $0.0241$ & $0.425$ & $103.8$ & $144$ \\
    ASAP-LS-0 ($L=32,\ M=4096$) & $5.67$ & $0.0201$ & $0.364$ & $104.2$ & $144$ \\
    ASAP-LS-0 ($L=64,\ M=4096$) & $6.55$ & $0.0177$ & $0.322$ & $104.9$ & $144$ \\
    ASAP-LS-0 ($L=128,\ M=4096$) & $8.66$ & $0.0158$ & $0.289$ & $106.6$ & $144$ \\
    \midrule
    ASAP-LS ($L=8,\ M=4096$) & $8.52$ & $0.0051$ & $0.066$ & $103.6$ & $144$ \\
    ASAP-LS ($L=16,\ M=4096$) & $8.75$ & $0.0040$ & $0.051$ & $103.8$ & $144$ \\
    ASAP-LS ($L=32,\ M=4096$) & $9.20$ & $0.0027$ & $0.035$ & $104.2$ & $144$ \\
    ASAP-LS ($L=64,\ M=4096$) & $10.08$ & $0.0018$ & $0.024$ & $104.9$ & $144$ \\
    ASAP-LS ($L=128,\ M=4096$) & $12.19$ & $\mathbf{0.0014}$ & $\mathbf{0.018}$ & $106.6$ & $144$ \\
    \midrule
    Sinkhorn normalizer ($S=3$) & $11.99$ & $0.0677$ & $1.527$ & -- & $144$ \\
    Sinkhorn normalizer ($S=5$) & $17.01$ & $0.0597$ & $1.430$ & -- & $144$ \\
    Sinkhorn teacher ($S=20$) & $54.65$ & -- & -- & -- & $144$ \\
    \bottomrule
  \end{tabular}
  \caption{ASAP-LS slice-count ablation on frozen IMDb Sinkhorn teacher ($S=20$) layers. Entries are means over held-out batch-layer cases from three trained teachers, with $M=4096$ fit examples. Teacher agreement is measured against the matched $S=20$ teacher, and timing comes from the dedicated slice-count sweep pass.}
  \label{tab:operator_ablation_appendix}
\end{table}

\subsection{Amortization Design Ablation}

Table~\ref{tab:amortization_design_appendix} compares design choices in the ASAP amortized layer on the frozen IMDb layers from the $S=20$ Sinkhorn teacher study. This study varies calibration size, the one-sided and two-sided $c$-transform, feature family, layer sharing, and direct dual prediction.

\begin{table}[!htbp]
  \centering
  \scriptsize
  \setlength{\tabcolsep}{3pt}
  \resizebox{\textwidth}{!}{%
  \begin{tabular}{lccccccc}
    \toprule
    Setting & Fit $M$ & Time (ms) $\downarrow$ & Output RMSE $\downarrow$ & Attention rel.\ $\ell_2$ $\downarrow$ & $\Delta_{\mathrm{row}\to 1}$ $\downarrow$ & $\Delta_{\mathrm{col}\to 1}$ $\downarrow$ & Speedup vs.\ $S=3$ \\
    \midrule
    \textbf{ASAP-0} & $512$ & $6.82 \pm 0.00$ & $0.020 \pm 0.017$ & $0.365 \pm 0.255$ & $0.514 \pm 0.053$ & $(4.12 \pm 0.09)\!\times\!10^{-7}$ & $1.76$ \\
    \textbf{ASAP-0} & $1024$ & $6.82 \pm 0.00$ & $0.020 \pm 0.017$ & $0.365 \pm 0.255$ & $0.514 \pm 0.053$ & $(4.12 \pm 0.09)\!\times\!10^{-7}$ & $1.76$ \\
    \textbf{ASAP-0} & $4096$ & $6.82 \pm 0.00$ & $0.020 \pm 0.017$ & $0.364 \pm 0.255$ & $0.514 \pm 0.053$ & $(4.12 \pm 0.09)\!\times\!10^{-7}$ & $1.76$ \\
    \textbf{ASAP} & $512$ & $10.43 \pm 0.00$ & $0.003 \pm 0.004$ & $0.035 \pm 0.070$ & $0.444 \pm 0.037$ & $(4.11 \pm 0.09)\!\times\!10^{-7}$ & $1.15$ \\
    \textbf{ASAP} & $1024$ & $10.43 \pm 0.00$ & $0.003 \pm 0.004$ & $\mathbf{0.035 \pm 0.070}$ & $0.444 \pm 0.037$ & $(4.11 \pm 0.09)\!\times\!10^{-7}$ & $1.15$ \\
    \textbf{ASAP} & $4096$ & $10.43 \pm 0.00$ & $\mathbf{0.003 \pm 0.004}$ & $0.035 \pm 0.070$ & $0.444 \pm 0.037$ & $(4.11 \pm 0.09)\!\times\!10^{-7}$ & $1.15$ \\
    \textbf{ASAP} shared coefficients & $4096$ & $10.43 \pm 0.00$ & $0.004 \pm 0.006$ & $0.051 \pm 0.094$ & $0.445 \pm 0.038$ & $(4.11 \pm 0.09)\!\times\!10^{-7}$ & $1.15$ \\
    \midrule
    Raw projections + ASAP transform & $4096$ & $9.70 \pm 0.00$ & $0.006 \pm 0.011$ & $0.077 \pm 0.167$ & $0.447 \pm 0.038$ & $(4.11 \pm 0.09)\!\times\!10^{-7}$ & $1.24$ \\
    Rank displacement + ASAP transform & $4096$ & $10.27 \pm 0.00$ & $0.006 \pm 0.012$ & $0.084 \pm 0.177$ & $0.449 \pm 0.039$ & $(4.11 \pm 0.09)\!\times\!10^{-7}$ & $1.17$ \\
    Predict both duals, no $c$-transform & $4096$ & $7.00 \pm 0.00$ & $0.297 \pm 1.191$ & $2.767 \pm 7.326$ & $0.775 \pm 0.809$ & $1.056 \pm 1.603$ & $1.71$ \\
    \midrule
    Sinkhorn normalizer ($S=3$) & -- & $11.99 \pm 0.00$ & $0.068 \pm 0.030$ & $1.527 \pm 0.319$ & $(3.33 \pm 0.23)\!\times\!10^{-7}$ & $0.833 \pm 0.140$ & $1.00$ \\
    Sinkhorn normalizer ($S=5$) & -- & $17.13 \pm 0.00$ & $0.060 \pm 0.015$ & $1.430 \pm 0.241$ & $(3.13 \pm 0.13)\!\times\!10^{-7}$ & $0.805 \pm 0.127$ & $0.70$ \\
    Sinkhorn teacher ($S=20$) & -- & $55.59 \pm 0.01$ & -- & -- & $0.441 \pm 0.038$ & $(2.70 \pm 0.23)\!\times\!10^{-7}$ & $0.22$ \\
    \bottomrule
  \end{tabular}
  }
  \caption{Frozen IMDb ablations with Sinkhorn teachers ($S=20$), reported as mean $\pm$ standard deviation over $144$ held-out batch-layer cases from three teacher seeds. ASAP-0 uses the one-sided $c$-transform, while ASAP uses the two-sided entropic $c$-transform. Feature controls replace sliced potentials before the same ASAP transform, and the direct-dual row removes the analytic $c$-transform. Timing and speedup are measured inside this design-ablation pass, and marginal errors follow Table~\ref{tab:operator_marginal_appendix}.}
  \label{tab:amortization_design_appendix}
\end{table}

\FloatBarrier
The controls support the ASAP parameterization under the same $S=20$ teacher setting used in the main frozen-layer study. ASAP-0 is the faster one-sided operator, while the default two-sided ASAP lowers output RMSE from $0.068$ for the $S=3$ normalizer to $0.003$ and remains faster than $S=3$. Shared coefficients remain usable with a small error increase, and sliced potentials outperform raw projections and rank-displacement features under the same transform. Predicting both duals is unstable without the analytic $c$-transform, with attention relative $\ell_2$ increasing to $2.767$ and active-column error increasing to $1.056$.

\end{document}